\documentclass[twoside]{article}

\usepackage[preprint]{twocolpreprint}

\usepackage[round]{natbib}

\usepackage{amsmath,amssymb,mathtools}
\usepackage{balance}
\usepackage{microtype}
\usepackage[utf8]{inputenc}
\usepackage[T1]{fontenc}
\usepackage{url}
\usepackage{booktabs}
\usepackage{amsfonts}
\usepackage{nicefrac}
\usepackage{microtype}
\usepackage{xcolor}
\usepackage{algorithm}
\usepackage{wrapfig} %
\usepackage{algpseudocode}
\usepackage{amsmath,amssymb,amsthm,mathtools}

\usepackage{hyperref}
\usepackage{placeins}

\newtheorem{theorem}{Theorem}
\newtheorem{lemma}[theorem]{Lemma}

\newcommand{\R}{\mathbb{R}}
\newcommand{\E}{\mathbb{E}}
\newcommand{\cA}{\mathcal{A}}
\newcommand{\cD}{\mathcal{D}}
\newcommand{\cR}{\mathcal{R}}
\newcommand{\cO}{\mathcal{O}}

\begin{document}

\runningtitle{Accelerating Linear Bandits with Persistent Candidate Sets}
\runningauthor{Tong, Korovaitceva, Sheshukova, Samsonov, Komiyama}

\twocolumn[

\aistatstitle{Accelerating Linear Bandits with Persistent Candidate Sets}

\aistatsauthor{Bingkui Tong$^{1}$ \quad Daria Korovaitceva$^{2}$ \quad Marina Sheshukova$^{2}$ \\[2pt] \bfseries Sergey Samsonov$^{2}$ \quad Junpei Komiyama$^{1,3}$}

\aistatsaddress{$^{1}$Mohamed bin Zayed University of Artificial Intelligence \\ $^{2}$HSE University \qquad $^{3}$RIKEN AIP \\[3pt]
\small\texttt{bingkui.tong@mbzuai.ac.ae} \quad \texttt{junpei@komiyama.info} \quad \texttt{\{dkorovaitseva,msheshukova,svsamsonov\}@hse.ru}}]

\begin{abstract}
Exhaustive optimistic scoring can make action selection costly in linear bandits with a large fixed set of $K$ actions, requiring $\mathcal{O}(K)$ score evaluations per round. We introduce Persistent Candidate LinUCB (\textsc{PC-LinUCB}), which maintains $k$ active actions and samples $q$ challengers uniformly from the remaining $K-k$ actions at each round. By scoring only these $k+q$ candidates, \textsc{PC-LinUCB} reduces the number of score evaluations per round to $\mathcal{O}(k+q)$.
We establish problem-independent and problem-dependent high-probability regret bounds consisting of a standard linear-bandit learning term and an additive discovery cost of order $\mathcal{O}\left((K-k)\log(K-k)/q\right)$ at fixed confidence.
Notably, this additional cost does not grow with the horizon.
For fixed gaps and sufficiently long horizons, we establish an $\Omega\left((K-k)/q\right)$ regret lower bound under uniform challenger sampling, showing that partial candidate access itself imposes an unavoidable regret cost.
Experiments show that \textsc{PC-LinUCB} substantially accelerates LinUCB while maintaining comparable regret, and that the persistent candidate mechanism can extend beyond LinUCB.
\end{abstract}

\section{Introduction}
\label{sec:introduction}

Bandit algorithms provide a general framework for sequential decision making under uncertainty, where the goal is to find the action (arm) of the largest expected reward.
The central challenge is balancing exploration of uncertain actions with exploitation of actions that appear promising. For a finite action set of size $K$, direct implementations of classical algorithms such as upper confidence bound (UCB) methods~\citep{LAI19854}, Thompson sampling~\citep{thompson1933likelihood,russo2018tutorial}, and LinUCB~\citep{li2010contextual} typically evaluate or compare all available actions at each round, making the action-selection cost scale linearly with $K$. This dependence can become prohibitively expensive in large-scale applications such as recommendation systems and online advertising, where the number of candidate actions can reach millions~\citep{sen2021topk,zhu2022largeaction}.

We study this large-$K$ computational challenge in the setting of linear bandits, where each action is represented by a feature vector in $\mathbb{R}^d$ and its mean reward is determined by an unknown parameter~\citep{dani2008stochastic,abbasi2011improved}. For a fixed action set of size $K$, a direct implementation of LinUCB with a dense confidence matrix requires $\mathcal{O}(d^2)$ computation to evaluate the uncertainty of one action, where $d$ is the feature dimension, leading to an $\mathcal{O}(Kd^2)$ scoring cost per round.

Existing work on computationally efficient linear bandits broadly follows two directions: reducing the cost of scoring each action, or reducing the number of actions that need to be evaluated. The former reduces the dimension-dependent computation through techniques such as random projection, matrix sketching, and low-rank approximation~\citep{yu2017cbrap,kuzborskij2019sketching,chen2020efficient,shustova2025scalable,wen2026revisiting}. The latter reduces the dependence on the action-set size through sublinear search, structured action selection, or subset-based approaches~\citep{jun2017scalable,yang2022sublinear,sen2021topk,zhu2022largeaction,hanna2023softbatch,lattimore2020features,bayati2020greedy,purohit2025case}.

In this work, we focus on reducing the number of actions evaluated. Reducing the number of scored actions, however, restricts access to the action set: an action excluded from the current candidate set cannot be selected at that round. Constructing an effective subset therefore requires balancing retention and discovery: promising actions should be retained once identified, while actions outside the current subset must still have opportunities to enter.

The optimistic scores of LinUCB provide a natural signal for identifying actions that remain potentially competitive. Motivated by this observation, we propose Persistent Candidate LinUCB (\textsc{PC-LinUCB}), which combines a persistent candidate set with UCB scores in the following way. At each round, it combines an active set of $k$ actions with $q$ challengers sampled uniformly from the remaining actions. A key innovation is our use of \emph{clipped scores} to guide both action selection and candidate retention: each candidate's score is updated to the minimum of its previous value and the newly evaluated LinUCB score. These scores are non-increasing while preserving optimism on the standard confidence event. Thus, once an action's clipped score falls below the optimal mean reward, it remains below that threshold, whereas an optimal action's clipped score never does. 
The action with the largest clipped score is selected, and the $k$ candidates with the largest clipped scores form the next active set. Scoring only these candidates reduces the number of score evaluations per round from $K$ to $k+q$, substantially lowering the cost when $k+q\ll K$.

By using the aforementioned property of the clipped UCB scores, we analyze and derive both problem-independent and problem-dependent regret upper bounds. In each case, the bound separates the standard linear-bandit learning term \citep{abbasi2011improved} from an additive discovery cost of order
\[
\cO\left(
    \frac{K-k}{q}\log(K-k)
\right).
\]
This term captures the time required for random challengers to expose an action that remains plausible for optimality and does not grow with the horizon at fixed confidence. We also establish a regret lower bound under uniform challenger sampling. Even if the identity of an initially inactive optimal action is known in advance, the policy cannot select it before it appears among the random challengers. For fixed-gap instances and sufficiently long horizons, this partial-access restriction alone incurs expected regret of order $\Omega((K-k)/q)$ when $q\ll K-k$.

We compare \textsc{PC-LinUCB} with five baselines on controlled synthetic problems and four large action sets derived from real data, and further apply the persistent candidate mechanism to SpectralUCB~\citep{valko2014spectral}. The results show substantial reductions in end-to-end wall-clock time with regret comparable to exhaustive LinUCB, alongside distinct time--regret tradeoffs for methods based on infrequent switching, matrix compression, and elimination. Ablations examine the total scoring budget and its allocation between persistent actions and random challengers. The SpectralUCB application further shows that the persistent candidate mechanism can extend beyond LinUCB.

Our contributions are as follows.
\begin{itemize}
    \item We introduce a persistent candidate mechanism guided by clipped scores for linear bandits with large fixed action sets. The clipped scores preserve optimism while remaining non-increasing and guide both action selection and candidate retention, enabling our analysis of partial candidate access.

    \item We establish problem-independent and problem-dependent regret upper bounds that separate the cost of learning the linear reward model from an additive cost induced by partial access to the action set.

    \item We prove a complementary regret lower bound under uniform challenger sampling, showing that partial candidate access itself imposes an unavoidable regret cost of order $\Omega((K-k)/q)$.

    \item We empirically demonstrate that \textsc{PC-LinUCB} substantially reduces computation while maintaining comparable regret, examine the tradeoff between candidate budget and performance, and show that the persistent candidate mechanism can extend beyond LinUCB through an application to SpectralUCB.
\end{itemize}

\section{Problem Formulation and Algorithm}
\label{sec:problem-algorithm}

\subsection{Problem Setting}

We consider a stochastic linear bandit with a fixed finite action set indexed by $[K]=\{1,\ldots,K\}$. Each action $i$ has a known feature vector $x_i\in\R^d$ that remains unchanged throughout the horizon. Its mean reward is determined by a shared unknown parameter $\theta^\star\in\R^d$:
\begin{equation}
    \mu_i = \langle x_i,\theta^\star\rangle,
    \qquad i\in[K].
    \label{eq:linear-mean}
\end{equation}
We assume $\lVert x_i\rVert_2\le L$, $\lVert\theta^\star\rVert_2\le S$, and $\mu_i\in[-1,1]$ for every action. At round $t$, the learner selects $I_t\in[K]$ and observes
\begin{equation}
    Y_t = \mu_{I_t}+\eta_t,
    \label{eq:reward-model}
\end{equation}
where $\eta_t$ is conditionally $R$-sub-Gaussian given the history before round $t$ and the selected action $I_t$. Let
\[
    i^\star\in\arg\max_{i\in[K]}\mu_i,
    \qquad
    \mu^\star=\mu_{i^\star},
    \qquad
    \Delta_i=\mu^\star-\mu_i.
\]
We measure performance by cumulative pseudo-regret:
\begin{equation}
    R_T=\sum_{t=1}^T\Delta_{I_t}.
    \label{eq:pseudo-regret}
\end{equation}

\subsection{Confidence Sets and Optimistic Scores}
\label{subsec:optimistic-scores}

We use the regularized least-squares estimator and ellipsoidal confidence construction of OFUL \citep{abbasi2011improved} to define the optimistic scores used by \textsc{PC-LinUCB}. For a regularization parameter $\lambda>0$, define
\vspace{-1em}
\begin{equation}
    \small
    V_t=\lambda I_d+\sum_{s=1}^tx_{I_s}x_{I_s}^{\top},\quad
    b_t=\sum_{s=1}^t x_{I_s}Y_s,\quad
    \widehat\theta_t=V_t^{-1}b_t.
    \label{eq:ridge-estimator}
\end{equation}
The initial values are $V_0=\lambda I_d$, $b_0=0$, and $\widehat\theta_0=0$. For a known horizon $T$ and confidence level $\delta\in(0,1)$, set
\begin{equation}
    \beta
    =
    R\sqrt{
        d\log\!\left(1+\frac{TL^2}{\lambda d}\right)
        +2\log\frac{1}{\delta}
    }
    +\sqrt{\lambda}S.
    \label{eq:confidence-radius}
\end{equation}
The associated confidence set is
\[
    \mathcal{C}_t
    =
    \left\{
        \theta\in\R^d:
        \lVert\theta-\widehat\theta_t\rVert_{V_t}\le\beta
    \right\},
\]
where $\lVert z\rVert_M=\sqrt{z^\top Mz}$. The self-normalized confidence bound gives
\begin{equation}
    \small
    \mathbb{P}(\mathcal{E}_\delta)\ge 1-\delta,
    \qquad
    \mathcal{E}_\delta
    =
    \{\theta^\star\in\mathcal{C}_t\text{ for all }0\le t\le T\}.
    \label{eq:confidence-event}
\end{equation}

For each action $i$, define the upper and lower confidence scores
\begin{equation}
    \small
    U_t(i)
    =
    \langle x_i,\widehat\theta_t\rangle
    +\beta\lVert x_i\rVert_{V_t^{-1}},
    L_t(i)
    =
    \langle x_i,\widehat\theta_t\rangle
    -\beta\lVert x_i\rVert_{V_t^{-1}},
    \label{eq:ordinary-confidence-scores}
\end{equation}
with confidence width
\begin{equation}
    w_t(i)=U_t(i)-L_t(i)=2\beta\lVert x_i\rVert_{V_t^{-1}}.
    \label{eq:confidence-width}
\end{equation}
On $\mathcal{E}_\delta$, $L_t(i)\le\mu_i\le U_t(i)$ holds simultaneously for every action and round. Our algorithm evaluates $U_t(i)$ only for actions in the current candidate set.

\subsection{Persistent Candidate LinUCB}
\label{subsec:persistent-candidates}

Fix an active-set size $k$ and a challenger budget $q$ such that
\[
    1\le k<K,
    \qquad
    1\le q\le K-k.
\]
The algorithm maintains an active set $\cA_t\subseteq[K]$ of size $k$. Choose any initial set $\cA_0$ with $|\cA_0|=k$ and initialize each action's clipped score as $\widetilde U(i)=U_0(i)$.

At round $t$, define the inactive set
\begin{equation}
    \mathcal{B}_{t-1}=[K]\setminus\cA_{t-1}.
    \label{eq:inactive-set}
\end{equation}
The algorithm samples $q$ challengers uniformly from $\mathcal{B}_{t-1}$ to form the challenger set $\cR_t$. Their union with the active set is the candidate set:
\begin{equation}
    \cD_t=\cA_{t-1}\cup\cR_t,
    \qquad |\cD_t|=k+q.
    \label{eq:candidate-set}
\end{equation}

Algorithm~\ref{alg:persistent-candidate} gives the complete procedure.

\begin{algorithm}[t]
\caption{Persistent Candidate LinUCB (\textsc{PC-LinUCB})}
\label{alg:persistent-candidate}
\begin{algorithmic}[1]
\Require Features $\{x_i\}_{i=1}^K$, horizon $T$, $k$, $q$, $\lambda$, $\beta$
\State $V_0\gets\lambda I_d$, $b_0\gets0$, $\widehat\theta_0\gets0$
\State Choose $\cA_0\subseteq[K]$ with $|\cA_0|=k$
\State $\widetilde U(i)\gets U_0(i)$ for every $i\in[K]$
\For{$t=1,\ldots,T$}
    \State $\mathcal{B}_{t-1}\gets[K]\setminus\cA_{t-1}$
    \State Draw $\cR_t$ uniformly among the $q$-subsets of $\mathcal{B}_{t-1}$
    \State $\cD_t\gets\cA_{t-1}\cup\cR_t$
    \For{$i\in\cD_t$}
        \State $U_{t-1}(i)\gets
        \langle x_i,\widehat\theta_{t-1}\rangle
        +\beta\lVert x_i\rVert_{V_{t-1}^{-1}}$
        \State $\widetilde U(i)\gets\min\{\widetilde U(i),U_{t-1}(i)\}$ \label{alg:refresh}
    \EndFor
    \State $I_t\in\arg\max_{i\in\cD_t}\widetilde U(i)$
    \State Play $I_t$ and observe $Y_t$
    \State Let $\cA_t$ be the $k$ actions in $\cD_t$ with largest
    $\widetilde U(i)$
    \State $V_t\gets V_{t-1}+x_{I_t}x_{I_t}^{\top}$
    \State $b_t\gets b_{t-1}+x_{I_t}Y_t$
    \State $\widehat\theta_t\gets V_t^{-1}b_t$
\EndFor
\end{algorithmic}
\end{algorithm}

\paragraph{Computational cost.} With a dense inverse covariance matrix, evaluating one optimistic score costs $\cO(d^2)$. Scoring only the $k+q$ actions in $\cD_t$ reduces the per-round scoring cost from $\cO(Kd^2)$ to $\cO((k+q)d^2)$. Model updates and candidate maintenance add no further dependence on $K$ per round.

The parameters $k$ and $q$ control retention and discovery, respectively: increasing $k$ retains more promising actions, while increasing $q$ searches the inactive set more quickly. The following sections bound the regret when $k+q\ll K$ and establish the discovery cost imposed by this candidate-access mechanism.

\section{Regret Analysis}
\label{sec:regret-analysis}

\subsection{Main Results}

We assume that at least one action is suboptimal and define $\Delta_{\min}=\min\{\Delta_i:\Delta_i>0\}$; otherwise, every action is optimal and the regret is trivially zero.

\begin{theorem}[Problem-independent regret]
\label{thm:problem-independent-regret}
Fix $T\ge 2$, $\delta,\delta'\in(0,1)$, $1\le k<K$, and $1\le q\le K-k$, and let $\beta$ be defined by Eq.~\eqref{eq:confidence-radius}. Then, with probability at least $1-\delta-\delta'$, Algorithm~\ref{alg:persistent-candidate} satisfies
\begin{equation}
    \begin{aligned}
        R_T
        &\le
        2\max\{\beta,1\}
        \sqrt{
            2Td\log\!\left(1+\frac{TL^2}{\lambda d}\right)
        }\\
        &\quad+
        2\left\lceil
            \frac{K-k}{q}
            \left(
                \log(K-k)+\log\frac{1}{\delta'}
            \right)
        \right\rceil.
    \end{aligned}
    \label{eq:problem-independent-regret}
\end{equation}
\end{theorem}

\begin{theorem}[Problem-dependent regret]
\label{thm:problem-dependent-regret}
Under the conditions of Theorem~\ref{thm:problem-independent-regret}, with probability at least $1-\delta-\delta'$,
Algorithm~\ref{alg:persistent-candidate} satisfies
\begin{equation}
    \begin{aligned}
        R_T
        &\le
        \frac{8\max\{\beta^2,1\}d}{\Delta_{\min}}
        \log\!\left(1+\frac{TL^2}{\lambda d}\right)\\
        &\quad+
        2\left\lceil
            \frac{K-k}{q}
            \left(
                \log(K-k)+\log\frac{1}{\delta'}
            \right)
        \right\rceil.
    \end{aligned}
    \label{eq:problem-dependent-regret}
\end{equation}
\end{theorem}

Note that the dependence on $\delta$ enters through the confidence radius $\beta$ in Eq.~\eqref{eq:confidence-radius}.
Both bounds consist of a standard linear-bandit estimation term and the same additive discovery term. At a fixed confidence level, the latter depends on $K$, $k$, and $q$, but not on the horizon $T$. 
It accounts for the regret induced by partial access to the action set.
Note that we can derive an expected regret by setting $\delta, \delta'=1/T$ and using the fact that the regret per round is constant (at most $2$).

\subsection{Proof Sketch}
\label{subsec:regret-proof-sketch}

We outline the main steps behind Theorems~\ref{thm:problem-independent-regret} and~\ref{thm:problem-dependent-regret}. The analysis separates the regret on rounds whose candidate sets contain an alive action from the cost of rounds whose candidate sets contain no alive action.

\paragraph{Alive actions.}
Work on the confidence event $\mathcal{E}_\delta$. Let $\widetilde U_t(i)$ denote the clipped score of action $i$ after round $t$, and define
\begin{equation}
    \mathcal{S}_t = \{i\in[K]:\widetilde U_t(i)\ge\mu^\star\},
    \qquad
    m_t = |\mathcal{S}_t\setminus\mathcal{A}_t|.
    \label{eq:proof-sketch-alive-set}
\end{equation}
Actions in $\mathcal{S}_t$ are called \emph{alive}. On $\mathcal{E}_\delta$, every UCB evaluated for an optimal action is at least $\mu^\star$. The minimum update of the clipped scores therefore ensures that an optimal action remains alive at every round. Moreover, clipped scores are non-increasing, so
\[
    \mathcal{S}_{t+1}\subseteq\mathcal{S}_t.
\]
Thus, once an action becomes dead, it remains dead. 

Define
\begin{equation}
    Z_t = \mathbf{1}\{\mathcal{D}_t\cap\mathcal{S}_t=\varnothing\},
    \qquad
    N_t^{\mathrm{empty}} = \sum_{u=1}^t Z_u.
    \label{eq:proof-sketch-empty-rounds}
\end{equation}
Hence, $N_t^{\mathrm{empty}}$ counts the rounds up to time $t$ for which the refreshed candidate set contains no alive action.
We call round $t$ an \emph{empty-candidate round} if $Z_t=1$, i.e., the refreshed candidate set contains no alive action, and a \emph{nonempty-candidate round} otherwise.

\paragraph{Empty-candidate rounds.}
Set $N=K-k$. Conditional on the history before $\mathcal{R}_t$ is sampled, any fixed set of $s$ inactive actions is missed by the challenger set with probability
\begin{equation}
    \frac{\binom{N-s}{q}}{\binom{N}{q}}
    \le
    \left(1-\frac{q}{N}\right)^s.
    \label{eq:proof-sketch-miss-probability}
\end{equation}
For $q<N$, set
\[
    p=\frac{q}{N},
    \qquad
    \rho=1-p.
\]
The number of inactive alive actions decreases in two ways. First, an inactive alive action becomes dead if it is sampled and its refreshed UCB falls below $\mu^\star$. Second, an alive candidate that survives the refresh may enter the active set through the top-$k$ update, thereby leaving the inactive alive set. To account for these changes together with the occurrence of empty-candidate rounds, we introduce the potential
\begin{equation}
    W_t = \rho^{-N_t^{\mathrm{empty}}}m_t,
    \label{eq:proof-sketch-potential}
\end{equation}
where $m_t$ tracks the remaining inactive alive actions and $\rho^{-N_t^{\mathrm{empty}}}$ accounts for the accumulated empty-candidate rounds. The key step is to show that their interaction makes $W_t$ a nonnegative supermartingale:
\begin{equation}
    \mathbb{E}[W_t\mid\mathcal{F}_{t-1}]
    \le W_{t-1}.
    \label{eq:proof-sketch-supermartingale}
\end{equation}

On $\mathcal{E}_\delta$, an optimal action remains alive. If round $t$ is an empty-candidate round, this action is not retained in $\mathcal{A}_t$, and hence $m_t\ge1$. Let $\tau_n$ denote the time of the $n$th empty-candidate round. On $\mathcal{E}_\delta\cap\{N_T^{\mathrm{empty}}\ge n\}$,
\[
    W_{\tau_n}
    =
    \rho^{-n}m_{\tau_n}
    \ge
    \rho^{-n}.
\]
The maximal inequality for nonnegative supermartingales, together with $W_0\le N$, therefore gives
\begin{equation}
    \mathbb{P}\!\left(
        \mathcal{E}_\delta
        \cap
        \{N_T^{\mathrm{empty}}\ge n\}
    \right)
    \le
    N\rho^n
    \le
    N\exp\!\left(-\frac{qn}{N}\right).
    \label{eq:proof-sketch-empty-tail}
\end{equation}
Taking
\[
    n
    =
    \left\lceil
        \frac{K-k}{q}
        \left(
            \log(K-k)+\log\frac{1}{\delta'}
        \right)
    \right\rceil
\]
gives, with joint failure probability at most $\delta'$ on $\mathcal{E}_\delta$,
\begin{equation}
    N_T^{\mathrm{empty}}
    \le
    \left\lceil
        \frac{K-k}{q}
        \left(
            \log(K-k)+\log\frac{1}{\delta'}
        \right)
    \right\rceil.
    \label{eq:proof-sketch-empty-bound}
\end{equation}
If $q=N$, all inactive actions are sampled at every round, so $N_T^{\mathrm{empty}}=0$ on $\mathcal{E}_\delta$.

\paragraph{Regret decomposition.}
Suppose $Z_t=0$. The refreshed candidate set then contains an alive action with clipped score at least $\mu^\star$. Since the algorithm selects the candidate with the largest clipped score, the selected action $I_t$ also has refreshed clipped score at least $\mu^\star$. This clipped score is at most the current UCB $U_{t-1}(I_t)$. Together with the lower confidence bound on $\mathcal{E}_\delta$,
\[
    \Delta_{I_t}
    \le
    U_{t-1}(I_t)-L_{t-1}(I_t)
    =
    w_{t-1}(I_t).
\]
Since $\mu_i\in[-1,1]$, we also have $\Delta_{I_t}\le2$. On an empty-candidate round we use this worst-case bound. Summing over rounds gives
\begin{equation}
    R_T
    \le
    \sum_{t=1}^T
    \min\{2,w_{t-1}(I_t)\}
    +2N_T^{\mathrm{empty}}.
    \label{eq:proof-sketch-gap-free-reduction}
\end{equation}
For the problem-dependent bound, every suboptimal action satisfies $\Delta_{I_t}\ge\Delta_{\min}$. Hence, on a nonempty-candidate round,
\[
    \Delta_{I_t}
    =
    \frac{\Delta_{I_t}^2}{\Delta_{I_t}}
    \le
    \frac{
        \min\{2,w_{t-1}(I_t)\}^2
    }{\Delta_{\min}},
\]
while an optimal action incurs zero regret. Therefore,
\begin{equation}
    R_T
    \le
    \frac{1}{\Delta_{\min}}
    \sum_{t=1}^T
    \min\{2,w_{t-1}(I_t)\}^2
    +2N_T^{\mathrm{empty}}.
    \label{eq:proof-sketch-gap-dependent-reduction}
\end{equation}

\paragraph{Confidence widths.}
Let
\[
    a_t = \|x_{I_t}\|_{V_{t-1}^{-1}}.
\]
The elliptical-potential lemma gives
\begin{equation}
    \sum_{t=1}^T\min\{1,a_t^2\}
    \le
    2d\log\!\left(1+\frac{TL^2}{\lambda d}\right).
    \label{eq:proof-sketch-elliptical-potential}
\end{equation}
Since $w_{t-1}(I_t)=2\beta a_t$,
\[
    \min\{2,w_{t-1}(I_t)\}
    \le
    2\max\{\beta,1\}\min\{1,a_t\}.
\]
Cauchy--Schwarz and Eq.~\eqref{eq:proof-sketch-elliptical-potential} then imply
\begin{equation}
    \small
    \sum_{t=1}^T\min\{2,w_{t-1}(I_t)\}
    \le
    2\max\{\beta,1\}
    \sqrt{
        2Td\log\!\left(1+\frac{TL^2}{\lambda d}\right)
    }.
    \label{eq:proof-sketch-linear-width}
\end{equation}
Similarly,
\begin{equation}
    \sum_{t=1}^T\min\{2,w_{t-1}(I_t)\}^2
    \le
    8\max\{\beta^2,1\}d
    \log\!\left(1+\frac{TL^2}{\lambda d}\right).
    \label{eq:proof-sketch-squared-width}
\end{equation}

Combining Eqs.~\eqref{eq:proof-sketch-empty-bound}, \eqref{eq:proof-sketch-gap-free-reduction}, and \eqref{eq:proof-sketch-linear-width} gives the problem-independent bound. Combining Eqs.~\eqref{eq:proof-sketch-empty-bound}, \eqref{eq:proof-sketch-gap-dependent-reduction}, and \eqref{eq:proof-sketch-squared-width} gives the problem-dependent bound. Since $\mathbb{P}(\mathcal{E}_\delta^c)\le\delta$ and the empty-candidate bound has joint failure probability at most $\delta'$ on $\mathcal{E}_\delta$, both bounds hold with probability at least $1-\delta-\delta'$.

\section{Tradeoff Between Acceleration and Regret}
\label{sec:acceleration-tradeoff}

Theorems~\ref{thm:problem-independent-regret} and~\ref{thm:problem-dependent-regret} separate the standard linear-bandit learning term from an additive term induced by partial candidate access. The latter is obtained by bounding $N_T^{\mathrm{empty}}$, the number of rounds in which the candidate set contains no alive action, and scales as $\widetilde{O}((K-k)/q)$. This section asks whether uniform partial access itself imposes an unavoidable regret cost with the same dependence on $K$, $k$, and $q$.

\begin{figure*}[htbp]
\centering
\includegraphics[width=\textwidth]{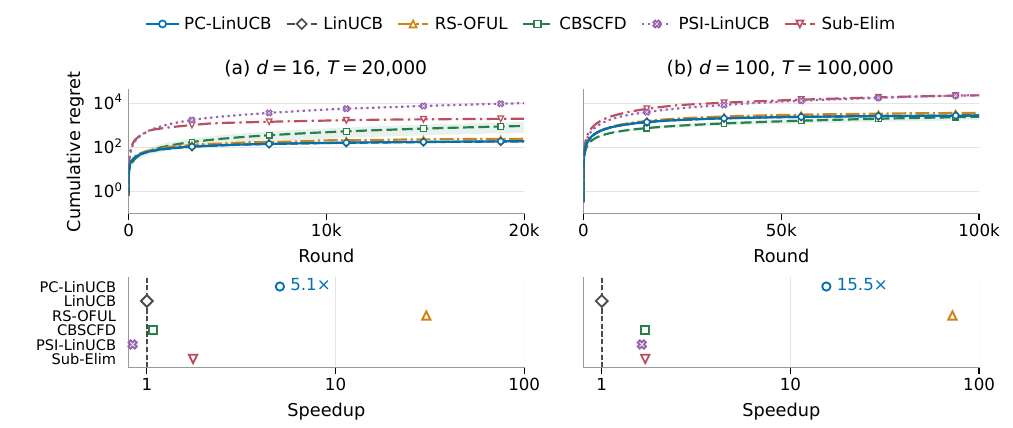}
\vspace{-2em}
\caption{Six-method comparison on synthetic action sets. Columns show $(d,T)=(16,20{,}000)$ and $(100,100{,}000)$, with $K=10{,}000$ and $k=q=250$ for \textsc{PC-LinUCB}. Top: cumulative pseudo-regret (symlog scale). Bottom: end-to-end speedup over LinUCB (logarithmic scale). Curves average ten repeats; bands show $\pm1$ SE.}
\label{fig:synthetic-results}
\end{figure*}

\subsection{Access-Induced Regret}
\label{subsec:access-induced-regret}

Consider a general class of policies that maintain $k$ active actions and sample $q$ challengers uniformly from the remaining $K-k$ actions at each round. The union of the active and challenger sets forms the candidate set, from which the policy must choose both the played action and the next active set. Apart from this access restriction, the policy is unrestricted. Algorithm~\ref{alg:persistent-candidate} belongs to this class.

Suppose the unique optimal action $i^\star$ is initially inactive, and let
\[
    \tau=\inf\{t\ge1:i^\star\in\mathcal{R}_t\}
\]
be the first round in which it appears in the challenger set. Before round $\tau$, the candidate-access mechanism prevents the policy from selecting $i^\star$, even if its identity is known in advance. Let $\Delta=\min_{i\ne i^\star}\Delta_i>0$ and define the \emph{access-induced regret}
\begin{equation}
    R_T^{\mathrm{access}}
    =
    \Delta\sum_{t=1}^T\mathbf{1}\{\tau>t\}.
    \label{eq:access-induced-regret}
\end{equation}
On $\{\tau>t\}$, the optimal action is unavailable and every action the policy can select is suboptimal, with gap at least $\Delta$. Hence, pathwise,
\begin{equation}
    R_T\ge R_T^{\mathrm{access}}.
    \label{eq:total-vs-access-regret}
\end{equation}

\begin{theorem}[Unavoidable Regret from Uniform Challenger Sampling]
\label{thm:uniform-gating-lower-bound}
Suppose the optimal action $i^\star$ is unique and initially inactive. Then, for every $T\ge1$, every policy in this class satisfies
\begin{equation}
    \E[R_T^{\mathrm{access}}]
    =
    \Delta\frac{K-k-q}{q}
    \left[
        1-\left(1-\frac{q}{K-k}\right)^T
    \right].
    \label{eq:exact-access-lower-bound}
\end{equation}
Moreover, if $q\le(K-k)/2$, then
\begin{equation}
    \E[R_T^{\mathrm{access}}]
    \ge
    \frac{1-e^{-1}}{2}\,
    \Delta
    \min\left\{
        T,\frac{K-k}{q}
    \right\}.
    \label{eq:order-access-lower-bound}
\end{equation}
\end{theorem}

Theorem~\ref{thm:uniform-gating-lower-bound} isolates the regret induced solely by the candidate-access restriction. Until its first appearance, the optimal action remains inactive and has probability $q/(K-k)$ of being sampled as a challenger in each round. Hence, $\tau$ is geometrically distributed. 

The access event $\mathcal{E}_t^{\mathrm{opt}}=\{i^\star\notin\mathcal{D}_t\}$ considered here is weaker than the empty-candidate event $\mathcal{E}_t^{\mathrm{empty}}=\{\mathcal{D}_t\cap\mathcal{S}_t=\varnothing\}$ used to derive the regret bounds in Theorems~\ref{thm:problem-independent-regret} and~\ref{thm:problem-dependent-regret}. On the confidence event $\mathcal{E}_\delta$, the optimal action is always alive, and therefore
\[
    \mathcal{E}_t^{\mathrm{empty}}
    \subseteq
    \mathcal{E}_t^{\mathrm{opt}}.
\]
The access event is weaker because the optimal action may be absent while the candidate set still contains suboptimal alive actions. Such rounds do not contribute to $N_T^{\mathrm{empty}}$, and their regret can still be controlled by the confidence-width argument. Thus, Theorem~\ref{thm:uniform-gating-lower-bound} does not provide a lower bound on $N_T^{\mathrm{empty}}$ or on the corresponding additive term itself. Instead, it shows that uniform partial access alone imposes an expected regret cost of order $\Omega((K-k)/q)$ for fixed $\Delta$ and sufficiently long horizons.

\begin{figure*}[t!]
\centering
\includegraphics[width=\textwidth]{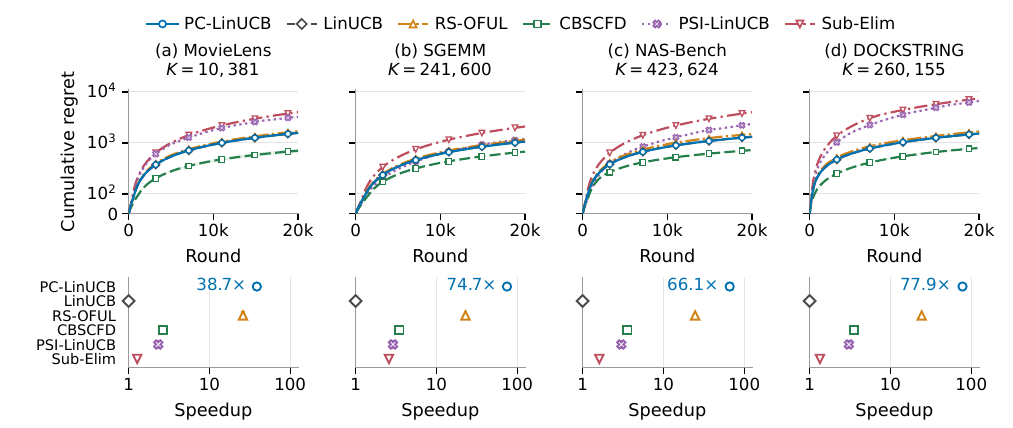}
\vspace{-2em}
\caption{Six-method comparison on real action features with linear simulated rewards, $d=128$, and $T=20{,}000$. \textsc{PC-LinUCB} uses $k=q$ and $B/K\approx2\%$. Top: cumulative pseudo-regret (symlog scale). Bottom: end-to-end speedup over LinUCB (logarithmic scale). Curves average five repeats; bands and error bars show $\pm1$ SE.}
\label{fig:semi-synthetic-results}
\end{figure*}

\subsection{Acceleration and Access-Induced Regret}
\label{subsec:acceleration-access-regret}

Uniform partial access reduces computation by scoring only $k+q$ candidate actions per round instead of all $K$ actions. Ignoring the per-action scoring cost common to both methods, we define the nominal acceleration factor as
\begin{equation}
    \mathsf{S}
    =
    \frac{K}{k+q}.
    \label{eq:nominal-acceleration}
\end{equation}
This quantity measures the reduction in the number of action scores evaluated per round; it does not include the model-update cost shared by exhaustive and candidate-based methods.

When $q\le(K-k)/2$ and $T\ge(K-k)/q$, Theorem~\ref{thm:uniform-gating-lower-bound} gives
\begin{equation}
    \E[R_T^{\mathrm{access}}]
    \ge
    \frac{1-e^{-1}}{2}\,
    \Delta\frac{K-k}{q}
    \ge
    \frac{1-e^{-1}}{2}\,
    \Delta\,\mathsf{S}.
    \label{eq:acceleration-regret-tradeoff}
\end{equation}
Thus, within this policy class, increasing the nominal acceleration is accompanied by a larger lower bound on the expected access-induced regret. This tradeoff applies in particular to our \textsc{PC-LinUCB} and arises from restricting the number of actions made available to the policy at each round.

\subsection{Active Set and Challenger Budget}
\label{subsec:roles-k-q}

The parameters $k$ and $q$ both contribute to the per-round scoring budget $k+q$, but they serve different roles in the candidate mechanism. Increasing $k$ retains more actions in the active set, whereas increasing $q$ explores the inactive set more aggressively and reduces the waiting time until an inactive optimal action becomes accessible. Accordingly, the access-induced lower bound depends directly on $(K-k)/q$, while the nominal acceleration depends on the total budget $k+q$. Section~\ref{subsec:ablation-experiments} examines these effects empirically by varying both the total candidate budget and its allocation between $k$ and $q$.

\begin{figure*}[t]
    \centering
    \includegraphics[width=\textwidth]{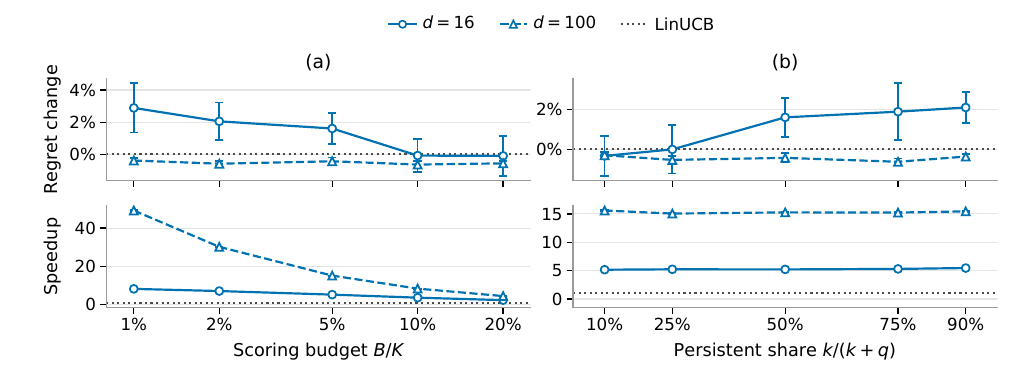}
    \vspace{-2em}
    \caption{\textsc{PC-LinUCB} ablations. (a) The total candidate budget $B=k+q$ varies with $k=q$. (b) The persistent share $k/B$ varies with $q=B-k$ and fixed $B/K=5\%$. Top: percentage change in mean final regret relative to LinUCB. Bottom: speedup over LinUCB.}
    \label{fig:ablation}
\end{figure*}

\section{Experiments}
\label{sec:experiments}

We compare cumulative pseudo-regret and wall-clock time, including initialization and all online rounds, on synthetic and four semi-synthetic benchmarks. The latter use real action features and linear rewards fitted to source labels, with Gaussian noise. Appendix~\ref{sec:experimental-details} gives experimental details; Appendix~\ref{app:detailed-speedup} reports timing and speedups.

All experiments use fixed action features, for which Sherman--Morrison caching reduces the per-action confidence cost of LinUCB from $\mathcal{O}(d^2)$ to $\mathcal{O}(d)$ but must refresh all $K$ cached widths every round, so it does not directly combine with candidate-based scoring. To isolate the gain from reducing the dependence on $K$, we use standard dense score evaluation for all methods.

\paragraph{Baselines.}
We compare \textsc{PC-LinUCB} with exhaustive LinUCB \citep{li2010contextual}, rarely switching OFUL (RS-OFUL; \citealp{abbasi2011improved}), CBSCFD \citep{chen2020efficient}, PSI-LinUCB \citep{shustova2025scalable}, and Sub-Elim \citep{yang2022sublinear}. All methods use $\lambda=0.2$ and the exploration coefficient $\beta$ defined in Eq.~\eqref{eq:confidence-radius}, with noise parameter $R=0.05$, norm bounds $L=S=1$, and confidence level $\delta=1/T$.

\subsection{Synthetic Experiments}
\label{subsec:synthetic-experiments}

\paragraph{Setup.} We sample the reward parameter and action vectors independently as
\[
    \theta^\star\sim\operatorname{Unif}(\mathbb{S}^{d-1}),
    \qquad
    x_i\overset{\mathrm{i.i.d.}}{\sim}\operatorname{Unif}(\mathbb{S}^{d-1}),
    \quad i\in[K],
\]
where $\mathbb{S}^{d-1}=\{x\in\mathbb{R}^d:\|x\|_2=1\}$, and add Gaussian reward noise of standard deviation $0.05$. We use $K=10{,}000$ and $(d,T)\in\{(16,20{,}000),(100,100{,}000)\}$, with ten paired repeats. For \textsc{PC-LinUCB}, we define the total candidate budget as $B=k+q$ and use $k=q=250$, giving $B/K=5\%$.

\paragraph{Results.}

In Figure~\ref{fig:synthetic-results}, \textsc{PC-LinUCB} speeds up LinUCB by $5.1\times$ and $15.5\times$ for $d=16$ and $d=100$, respectively, while changing mean final regret by only $+1.61\%$ and $-0.42\%$. RS-OFUL achieves larger speedups of $30.3\times$ and $72.3\times$, but increases mean final regret by about $31\%$. At $d=100$, CBSCFD achieves lower mean regret than \textsc{PC-LinUCB}, but at substantially higher computational cost.

\subsection{Semi-Synthetic Experiments}
\label{subsec:semi-synthetic-experiments}

\paragraph{Setup.} We use MovieLens-20M~\citep{vig2012taggenome, harper2015movielens}, UCI SGEMM~\citep{paredes2017sgemm}, NAS-Bench-101~\citep{ying2019nasbench101}, and DOCKSTRING~\citep{garciaortegon2022dockstring}, with $K=10{,}381$, $241{,}600$, $423{,}624$, and $260{,}155$, respectively. Each action set has $d=128$ features. Each benchmark uses $T=20{,}000$ and five paired repeats, with online rewards generated as 
\[
    Y_t=\mu_{I_t}+\eta_t,
    \qquad
    \eta_t\sim\mathcal{N}(0,0.05^2).
\]
\textsc{PC-LinUCB} uses $k=q=\lceil0.01K\rceil$.

\paragraph{Results.} In Figure~\ref{fig:semi-synthetic-results}, \textsc{PC-LinUCB} is the fastest of the six methods, with $38.7$--$77.9\times$ speedups over LinUCB and mean final regret differing by at most $1.2\%$. CBSCFD achieves lower regret on all four datasets, at $44.6$--$63.0\%$ of LinUCB's regret, but takes $14.6$--$22.1$ times as long as \textsc{PC-LinUCB}.

\subsection{Ablation Experiments}
\label{subsec:ablation-experiments}
Both ablations use the synthetic settings with ten paired repeats per configuration, varying only the candidate budget or its allocation.

\paragraph{Candidate budget.} We vary $B/K\in\{1\%,2\%,5\%,10\%,20\%\}$ with $k=q=B/2$ in both synthetic settings. Figure~\ref{fig:ablation} (a) shows that increasing the budget reduces speedup from $8.2\times$ to $2.4\times$ for $d=16$ and from $49.2\times$ to $4.5\times$ for $d=100$. The mean regret penalty for $d=16$ declines from $2.9\%$ to approximately zero, while regret at $d=100$ varies little.

\paragraph{Budget allocation.} We fix the total candidate budget at $B/K=5\%$ and vary the persistent share $k/B\in\{0.1,0.25,0.5,0.75,0.9\}$, with $q=B-k$. Since the total scoring budget is fixed, the speedup remains nearly unchanged across allocations. Figure~\ref{fig:ablation} (b) shows that, for $d=16$, mean final regret increases as a larger fraction of the budget is assigned to the active set, whereas for $d=100$ it remains nearly unchanged.

\subsection{Application to Spectral Bandits}
\label{subsec:spectral-application}

To examine whether the persistent candidate mechanism extends beyond LinUCB, we apply it to SpectralUCB~\citep{valko2014spectral}  on a synthetic graph bandit. We modify only the action-scoring procedure, replacing exhaustive scoring with the same persistent candidate mechanism used in \textsc{PC-LinUCB}. We consider $K=10{,}000$ graph nodes over $T=50{,}000$ rounds and compare exhaustive SpectralUCB with three candidate budgets, $B/K\in\{1\%,5\%,10\%\}$, using $k=q=B/2$. Each configuration is evaluated over ten paired repeats.

Figure~\ref{fig:spectral-application} shows that all three candidate budgets achieve nearly identical cumulative pseudo-regret to exhaustive SpectralUCB while providing substantial end-to-end speedups. This result illustrates that the persistent candidate mechanism can also serve as an efficient action-selection mechanism beyond LinUCB.

\begin{figure}[t]
    \centering
    \includegraphics[width=\columnwidth]{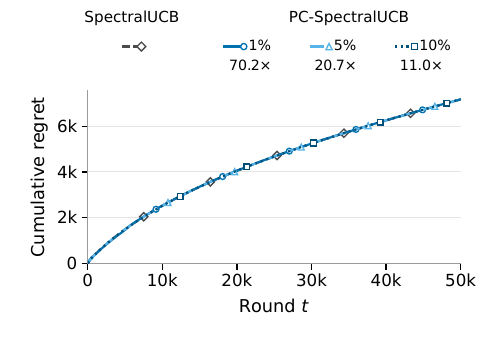}
    \vspace{-3em}
    \caption{Persistent candidate mechanism applied to SpectralUCB. Curves show mean cumulative pseudo-regret over ten runs for $B/K\in\{1\%,5\%,10\%\}$ with $k=q$; labels show speedup over SpectralUCB.}
    \label{fig:spectral-application}
\end{figure}

\section{Limitations}
\label{sec:discussion}

Our method relies on a fixed action set with fixed features, allowing clipped scores and persistent candidates to remain valid across rounds. It does not directly extend to contextual bandits, where changing actions or features can invalidate previously stored scores and candidates.

\clearpage
\section*{AI Use Statement}
\label{app:ai_use_statement}

We used AI tools to assist with code development, experiment execution and visualization, as well as language polishing of the manuscript. AI tools were also used to review the theoretical arguments for potential errors. In particular, an earlier version of Lemma~\ref{lem:appendix-inactive-alive-potential} and its proof was incorrect, and GPT-6 Pro assisted in identifying the issue and developing the corrected proof presented in this paper. All AI-assisted outputs, including experimental code and theoretical arguments, were reviewed and verified by the authors.

\balance
\bibliography{references}

\clearpage
\appendix
\onecolumn

\noindent\textbf{Appendix Contents.}
\begin{itemize}
    \setlength{\itemsep}{0pt}
    \setlength{\parskip}{0pt}
    \setlength{\parsep}{0pt}
    \item Appendix~\ref{sec:appendix-regret-upper-bounds}: Proofs of the Regret Upper Bounds
    \item Appendix~\ref{sec:appendix-uniform-challenger-lower-bound}: Proof of the Access-Induced Regret Lower Bound
    \item Appendix~\ref{sec:experimental-details}: Experimental Details
    \begin{itemize}
        \setlength{\itemsep}{0pt}
        \setlength{\parskip}{0pt}
        \setlength{\parsep}{0pt}
        \item Appendix~\ref{app:data_construction}: Data Construction
        \item Appendix~\ref{app:method_parameters}: Method Parameters
        \item Appendix~\ref{app:spectral-experiment}: Spectral Bandit Application
        \item Appendix~\ref{app:timing}: Timing
        \item Appendix~\ref{app:detailed-speedup}: Detailed Timing and Speedup Results
    \end{itemize}
\end{itemize}

\section{Proofs of the Regret Upper Bounds}
\label{sec:appendix-regret-upper-bounds}

Let $\overline U_{t-1}^{-}(i)$ and $\overline U_{t-1}^{+}(i)$ denote the values of the clipped score $\widetilde U(i)$ immediately before and after the refresh in round $t$ (i.e. Step~\ref{alg:refresh} in Algorithm~\ref{alg:persistent-candidate}), respectively. Thus $\overline U_0^{-}(i)=U_0(i)$ and
\begin{equation}
    \overline U_{t-1}^{+}(i)
    =
    \begin{cases}
        \min\{\overline U_{t-1}^{-}(i),U_{t-1}(i)\},
        & i\in\mathcal{D}_t,\\
        \overline U_{t-1}^{-}(i),
        & i\notin\mathcal{D}_t.
    \end{cases}
    \label{eq:appendix-lazy-refresh}
\end{equation}
The algorithm selects an action and retains the next active set according to
\begin{equation}
    I_t\in\arg\max_{i\in\mathcal{D}_t}\overline U_{t-1}^{+}(i),
    \qquad
    \mathcal{A}_t
    \in
    \arg\max_{\substack{\mathcal{A}\subseteq\mathcal{D}_t\\|\mathcal{A}|=k}}
    \sum_{i\in\mathcal{A}}\overline U_{t-1}^{+}(i).
    \label{eq:appendix-lazy-selection-retention}
\end{equation}
After round $t$, these scores are carried into the next round as
\begin{equation}
    \overline U_t^{-}(i)=\overline U_{t-1}^{+}(i),
    \qquad i\in[K].
    \label{eq:appendix-lazy-carry}
\end{equation}
Every refreshed candidate score is at most its current UCB score:
\begin{equation}
    \overline U_{t-1}^{+}(i)\le U_{t-1}(i),
    \qquad i\in\mathcal{D}_t.
    \label{eq:appendix-lazy-score-sandwich}
\end{equation}

Define the alive set from the clipped scores after round $t$, and let $m_t$ be the number of inactive alive actions:
\begin{equation}
    \mathcal{S}_t=\{i\in[K]:\overline U_t^{-}(i)\ge\mu^\star\},
    \qquad
    m_t=|\mathcal{S}_t\setminus\mathcal{A}_t|.
    \label{eq:appendix-alive-set}
\end{equation}
On the confidence event $\mathcal{E}_\delta$ from Eq.~\eqref{eq:confidence-event}, every evaluated $U_s(i)$ upper bounds $\mu_i$, so $i^\star\in\mathcal{S}_t$ for every $t$. Since the clipped scores never increase,
\begin{equation}
    \mathcal{S}_{t+1}\subseteq\mathcal{S}_t.
    \label{eq:appendix-alive-set-nesting}
\end{equation}
An action that becomes dead therefore remains dead.

Define the empty-candidate indicator after the round-$t$ refresh and its cumulative count by
\begin{equation}
    Z_t=\mathbf{1}\{\mathcal{D}_t\cap\mathcal{S}_t=\varnothing\},
    \qquad
    N_t^{\mathrm{empty}}=\sum_{u=1}^t Z_u,
    \qquad N_0^{\mathrm{empty}}=0.
    \label{eq:appendix-empty-candidate-count}
\end{equation}
We call round $t$ an \emph{empty-candidate round} if $Z_t=1$, i.e., the refreshed candidate set contains no alive action, and a \emph{nonempty-candidate round} otherwise.
When $\mathcal{D}_t$ contains an alive action after refresh, the selected clipped score is at least $\mu^\star$, allowing us to bound instantaneous regret by the current confidence width. For the rounds counted by $N_T^{\mathrm{empty}}$, we use the worst-case regret bound.

\begin{lemma}[Regret reduction]
\label{lem:appendix-regret-reduction}
On $\mathcal{E}_\delta$,
\begin{align}
    R_T
    &\le
    \sum_{t=1}^T\min\{2,w_{t-1}(I_t)\}
    +2N_T^{\mathrm{empty}},
    \label{eq:appendix-gap-free-reduction}\\
    R_T
    &\le
    \frac{1}{\Delta_{\min}}
    \sum_{t=1}^T\min\{2,w_{t-1}(I_t)\}^2
    +2N_T^{\mathrm{empty}},
    \label{eq:appendix-gap-dependent-reduction}
\end{align}
where the second inequality applies under the assumptions of Theorem~\ref{thm:problem-dependent-regret}.
\end{lemma}

\begin{proof}
Fix a round $t$ with $Z_t=0$, and choose $j\in\mathcal{D}_t\cap\mathcal{S}_t$. Equations~\eqref{eq:appendix-lazy-carry} and~\eqref{eq:appendix-alive-set} give $\overline U_{t-1}^{+}(j)=\overline U_t^{-}(j)\ge\mu^\star$. Since the algorithm maximizes the refreshed clipped score over $\mathcal{D}_t$,
\begin{equation}
    \overline U_{t-1}^{+}(I_t)\ge\overline U_{t-1}^{+}(j)\ge\mu^\star.
    \label{eq:appendix-selected-score-optimistic}
\end{equation}
Combining Eq.~\eqref{eq:appendix-lazy-score-sandwich} for $I_t$ with the lower confidence bound on $\mathcal{E}_\delta$ gives
\begin{equation}
    \Delta_{I_t}
    =\mu^\star-\mu_{I_t}
    \le\overline U_{t-1}^{+}(I_t)-\mu_{I_t}
    \le U_{t-1}(I_t)-L_{t-1}(I_t)
    =w_{t-1}(I_t).
    \label{eq:appendix-one-round-width-bound}
\end{equation}
Because all mean rewards lie in $[-1,1]$, $\Delta_{I_t}\le2$. Therefore, on a nonempty-candidate round ($Z_t=0$), $\Delta_{I_t}\le\min\{2,w_{t-1}(I_t)\}$; on an empty-candidate round ($Z_t=1$), the bound $\Delta_{I_t}\le2$ remains valid. Summing these two cases proves Eq.~\eqref{eq:appendix-gap-free-reduction}. For a suboptimal $I_t$, $\Delta_{I_t}\ge\Delta_{\min}$, so on a nonempty-candidate round,
\begin{equation}
    \Delta_{I_t}
    =\frac{\Delta_{I_t}^2}{\Delta_{I_t}}
    \le
    \frac{\min\{2,w_{t-1}(I_t)\}^2}{\Delta_{\min}}.
    \label{eq:appendix-one-round-squared-width-bound}
\end{equation}
When $I_t$ is optimal, its regret is zero and the squared-width bound also holds. Summing over rounds proves Eq.~\eqref{eq:appendix-gap-dependent-reduction}.
\end{proof}

The two bounds in Lemma~\ref{lem:appendix-regret-reduction} will be used separately below:
Eq.~\eqref{eq:appendix-gap-free-reduction} yields the problem-independent bound in Theorem~\ref{thm:problem-independent-regret},
whereas Eq.~\eqref{eq:appendix-gap-dependent-reduction}, together with the minimum gap $\Delta_{\min}$,
yields the problem-dependent bound in Theorem~\ref{thm:problem-dependent-regret}.

Set $N=K-k$, and let $\mathcal{F}_{t-1}$ denote the history available immediately before $\mathcal{R}_t$ is drawn. This history determines the current scores $U_{t-1}(i)$, including those the algorithm does not evaluate. The next lemma bounds the probability that the challenger set misses a specified subset of inactive actions.

\begin{lemma}[One-step alive-action hit probability]
\label{lem:appendix-one-step-hit}
Suppose $q<N$, and let $\mathcal{G}_t\subseteq[K]\setminus\mathcal{A}_{t-1}$ be an $\mathcal{F}_{t-1}$-measurable set with $|\mathcal{G}_t|=s$. Then, conditionally on $\mathcal{F}_{t-1}$,
\begin{equation}
    \mathbb{P}(\mathcal{R}_t\cap\mathcal{G}_t=\varnothing\mid\mathcal{F}_{t-1})
    =a_s
    :=\frac{\binom{N-s}{q}}{\binom{N}{q}}
    \le\left(1-\frac{q}{N}\right)^s,
    \label{eq:appendix-one-step-miss}
\end{equation}
where $\binom{N-s}{q}=0$ when $N-s<q$.
\end{lemma}

\begin{proof}
Since $\mathcal{R}_t$ is a uniformly sampled $q$-subset of the $N$ inactive actions, the exact miss probability is the ratio in Eq.~\eqref{eq:appendix-one-step-miss}. If $N-s<q$, the bound is immediate. Otherwise,
\begin{equation}
    \frac{\binom{N-s}{q}}{\binom{N}{q}}
    =\prod_{j=0}^{s-1}\frac{N-q-j}{N-j}
    =\prod_{j=0}^{s-1}\left(1-\frac{q}{N-j}\right)
    \le\left(1-\frac{q}{N}\right)^s,
    \label{eq:appendix-hypergeometric-bound}
\end{equation}
with an empty product equal to one.
\end{proof}

We bound the number of empty-candidate rounds using a potential based on the inactive alive actions. The potential accounts for both actions that die upon refresh and those retained in the active set.
The argument uses the fact that clipped scores never increase, so a dead action never becomes alive again and the alive sets are nested as in Eq.~\eqref{eq:appendix-alive-set-nesting}.

\begin{lemma}[Inactive-alive potential]
\label{lem:appendix-inactive-alive-potential}
Suppose $q<N$. Let $p=q/N$ be the probability that a given
inactive action is sampled, and set $\rho=1-p$.
Then
\begin{equation}
    W_t=\rho^{-N_t^{\mathrm{empty}}}m_t,
    \qquad t=0,1,\ldots,T,
    \label{eq:appendix-inactive-alive-potential}
\end{equation}
is a nonnegative supermartingale with $W_0=m_0\le N$.
\end{lemma}

\begin{proof}
Fix a round $t$ and condition on $\mathcal{F}_{t-1}$. Write
$m=m_{t-1}$ for the number of inactive alive actions before
$\mathcal{R}_t$ is sampled. Let $x$ be the number of actions in
$\mathcal{A}_{t-1}$ that remain alive after the round-$t$ refresh.
Since every action in $\mathcal{A}_{t-1}$ is refreshed, $x$ is
determined by $\mathcal{F}_{t-1}$.

Partition the $m$ inactive alive actions into $s$ actions satisfying
$U_{t-1}(i)\ge\mu^\star$ and $r=m-s$ actions satisfying
$U_{t-1}(i)<\mu^\star$. Thus, an action in the first group remains
alive if sampled and refreshed, whereas an action in the second group
becomes dead if sampled and refreshed. Both groups are fixed before
$\mathcal{R}_t$ is sampled. Let $h$ and $\ell$ be their respective
counts in $\mathcal{R}_t$.

The refresh removes exactly $\ell$ of the initially inactive alive
actions. Among the candidates, exactly $x+h$ actions are alive.
Since every alive candidate has clipped score at least $\mu^\star$
and every dead candidate has clipped score below $\mu^\star$, the
top-$k$ update retains exactly $\min\{k,x+h\}$ alive actions.
Therefore,
\begin{equation}
    m_t=m+x-\ell-\min\{k,x+h\},
    \qquad
    Z_t=\mathbf{1}\{x+h=0\}.
    \label{eq:appendix-inactive-alive-update}
\end{equation}

\paragraph{Case 1: $x\ge1$.}
Then $Z_t=0$, and since $\min\{k,x+h\}\ge x$,
Eq.~\eqref{eq:appendix-inactive-alive-update} gives $m_t\le m$.
Hence
\begin{equation}
    \mathbb{E}[\rho^{-Z_t}m_t\mid\mathcal{F}_{t-1}]
    \le m.
    \label{eq:appendix-potential-one-step}
\end{equation}

\paragraph{Case 2: $x=0$.}
In this case $Z_t=\mathbf{1}\{h=0\}$. Write
$a=\mathbb{P}(h=0\mid\mathcal{F}_{t-1})$.
Lemma~\ref{lem:appendix-one-step-hit} gives $a\le\rho^s$. For
$s\ge1$,
\begin{equation}
    psa\le ps\rho^s\le\rho(1-\rho^s)\le\rho(1-a),
    \label{eq:appendix-hit-potential-balance}
\end{equation}
because
$1-\rho^s=p\sum_{j=0}^{s-1}\rho^j\ge ps\rho^{s-1}$.
The same inequality holds for $s=0$, since then $a=1$.

Uniform sampling gives
$\mathbb{E}[\ell\mid\mathcal{F}_{t-1}]=pr$. Moreover,
Eq.~\eqref{eq:appendix-inactive-alive-update} becomes
$m_t=m-\ell-\min\{k,h\}$. Since
$\min\{k,h\}\ge\mathbf{1}\{h\ge1\}$,
\begin{equation}
    \mathbb{E}[m_t\mid\mathcal{F}_{t-1}]
    \le m-pr-(1-a).
    \label{eq:appendix-inactive-alive-drift}
\end{equation}

On $\{Z_t=1\}=\{h=0\}$, all $s$ actions in the first group remain
inactive and alive. For each of the other $r$ actions, contributing
to $Z_tm_t$ requires that it not be sampled, which has probability
$\rho$. Hence,
\begin{equation}
    \mathbb{E}[Z_tm_t\mid\mathcal{F}_{t-1}]
    \le sa+r\rho.
    \label{eq:appendix-empty-alive-moment}
\end{equation}

Since $\rho^{-Z_t}=1+(p/\rho)Z_t$,
Eqs.~\eqref{eq:appendix-hit-potential-balance}--\eqref{eq:appendix-empty-alive-moment}
give
\begin{align}
    \mathbb{E}[\rho^{-Z_t}m_t\mid\mathcal{F}_{t-1}]
    &\le m-pr-(1-a)+\frac{p}{\rho}(sa+r\rho)
    \nonumber\\
    &=m-(1-a)+\frac{ps}{\rho}a
    \le m.
    \label{eq:appendix-potential-drift}
\end{align}
Thus Eq.~\eqref{eq:appendix-potential-one-step} holds in both cases.

Since
$N_t^{\mathrm{empty}}=N_{t-1}^{\mathrm{empty}}+Z_t$
and $\rho^{-N_{t-1}^{\mathrm{empty}}}$ is
$\mathcal{F}_{t-1}$-measurable,
\begin{align}
    \mathbb{E}[W_t\mid\mathcal{F}_{t-1}]
    &=
    \rho^{-N_{t-1}^{\mathrm{empty}}}
    \mathbb{E}[\rho^{-Z_t}m_t\mid\mathcal{F}_{t-1}]
    \nonumber\\
    &\le
    \rho^{-N_{t-1}^{\mathrm{empty}}}m_{t-1}
    =W_{t-1}.
\end{align}
Since $W_t\ge0$, $(W_t)_{t=0}^T$ is therefore a nonnegative
supermartingale. Finally,
$W_0=m_0\le |[K]\setminus\mathcal{A}_0|=N$.
\end{proof}

On $\mathcal{E}_\delta$, an empty-candidate round leaves the optimal action alive and inactive, so $m_t\ge1$ whenever such a round occurs. We use this observation to convert the potential bound into a bound on the total number of empty-candidate rounds.

\begin{lemma}[Number of empty-candidate rounds]
\label{lem:appendix-empty-candidate-bound}
For every $\delta'\in(0,1)$,
\begin{equation}
    \mathbb{P}\!\left(
        \mathcal{E}_\delta
        \cap
        \left\{
            N_T^{\mathrm{empty}}
            >
            \left\lceil
                \frac{K-k}{q}
                \left(\log(K-k)+\log\frac{1}{\delta'}\right)
            \right\rceil
        \right\}
    \right)
    \le\delta'.
    \label{eq:appendix-empty-candidate-high-probability}
\end{equation}
\end{lemma}

\begin{proof}
First suppose $q<N$ and fix an integer $n\ge1$. On $\mathcal{E}_\delta\cap\{N_T^{\mathrm{empty}}\ge n\}$, let $\tau_n\le T$ be the $n$th empty-candidate round. Then $m_{\tau_n}\ge1$, giving
\begin{equation}
    W_{\tau_n}
    =\rho^{-n}m_{\tau_n}
    \ge\rho^{-n}.
    \label{eq:appendix-empty-potential-threshold}
\end{equation}
For any nonnegative supermartingale $(W_t)$ and any $a>0$, the
maximal inequality states that
$\mathbb{P}(\max_{0\le t\le T}W_t\ge a)\le\mathbb{E}[W_0]/a$.
Indeed, stopping at the first crossing of $a$ (or at $T$ if no crossing
occurs) gives a stopped value at least $a$ on the crossing event and
nonnegative otherwise, while its expectation is at most $\mathbb{E}[W_0]$.
Here, Eq.~\eqref{eq:appendix-empty-potential-threshold} implies that
$n$ empty-candidate rounds on $\mathcal{E}_\delta$ force such a crossing
with $a=\rho^{-n}$. Since $W_0\le N$, we obtain
\begin{align}
    \mathbb{P}\!\left(\mathcal{E}_\delta\cap\{N_T^{\mathrm{empty}}\ge n\}\right)
    &\le
    \mathbb{P}\!\left(\max_{0\le t\le T}W_t\ge\rho^{-n}\right)
    \nonumber\\
    &\le N\rho^n
    \le N\exp\!\left(-\frac{qn}{N}\right).
    \label{eq:appendix-empty-exponential-tail}
\end{align}
Here, the last inequality uses $\rho=1-q/N$ and $1-x\le e^{-x}$:
$\rho^n=(1-q/N)^n\le e^{-qn/N}$.
Taking
\begin{equation}
    n=\left\lceil\frac{N}{q}\left(\log N+\log\frac{1}{\delta'}\right)\right\rceil
    \label{eq:appendix-empty-threshold}
\end{equation}
makes the right-hand side of Eq.~\eqref{eq:appendix-empty-exponential-tail} at most $\delta'$. This proves the claim for $q<N$. If $q=N$, then $\mathcal{D}_t=[K]$ at every round. On $\mathcal{E}_\delta$, the optimal action is always alive, so $N_T^{\mathrm{empty}}=0$. The bound follows in this case as well.
\end{proof}

Lemma~\ref{lem:appendix-regret-reduction} bounds regret on nonempty-candidate rounds through the confidence widths of the played actions. We now bound these widths using the elliptical potential.

Let
\begin{equation}
    a_t=\lVert x_{I_t}\rVert_{V_{t-1}^{-1}}.
    \label{eq:appendix-elliptical-radius}
\end{equation}
The standard elliptical-potential lemma \citep[Lemma~11]{abbasi2011improved} gives
\begin{equation}
    \sum_{t=1}^T\min\{1,a_t^2\}
    \le
    2\log\frac{\det(V_T)}{\det(V_0)}
    \le
    2d\log\!\left(1+\frac{TL^2}{\lambda d}\right).
    \label{eq:appendix-elliptical-potential}
\end{equation}

The problem-independent bound uses the linear sum of capped confidence widths.

\begin{proof}[Proof of Theorem~\ref{thm:problem-independent-regret}]
Since $w_{t-1}(I_t)=2\beta a_t$,
\begin{equation}
    \min\{2,w_{t-1}(I_t)\}
    \le
    2\max\{\beta,1\}\min\{1,a_t\}.
    \label{eq:appendix-capped-width-linear}
\end{equation}
Summing over $t$ and applying the Cauchy--Schwarz inequality gives
\begin{align}
    \sum_{t=1}^T\min\{2,w_{t-1}(I_t)\}
    &\le
    2\max\{\beta,1\}
    \sum_{t=1}^T\min\{1,a_t\}
    \nonumber\\
    &\le
    2\max\{\beta,1\}
    \sqrt{
        T\sum_{t=1}^T\min\{1,a_t^2\}
    }.
\end{align}
Together with Eq.~\eqref{eq:appendix-elliptical-potential}, this yields
\begin{equation}
    \sum_{t=1}^T\min\{2,w_{t-1}(I_t)\}
    \le
    2\max\{\beta,1\}
    \sqrt{
        2Td\log\!\left(1+\frac{TL^2}{\lambda d}\right)
    }.
    \label{eq:appendix-gap-free-width-sum}
\end{equation}

By Lemma~\ref{lem:appendix-empty-candidate-bound}, on
$\mathcal{E}_\delta$ except on an event of probability at most
$\delta'$,
\begin{equation}
    N_T^{\mathrm{empty}}
    \le
    \left\lceil
        \frac{K-k}{q}
        \left(
            \log(K-k)+\log\frac{1}{\delta'}
        \right)
    \right\rceil .
    \label{eq:appendix-empty-candidate-bound-use}
\end{equation}
Combining Eq.~\eqref{eq:appendix-gap-free-width-sum} and
Eq.~\eqref{eq:appendix-empty-candidate-bound-use} with the
problem-independent regret reduction
Eq.~\eqref{eq:appendix-gap-free-reduction} gives
\begin{align}
    R_T
    &\le
    2\max\{\beta,1\}
    \sqrt{
        2Td\log\!\left(1+\frac{TL^2}{\lambda d}\right)
    }
    \nonumber\\
    &\qquad
    +2
    \left\lceil
        \frac{K-k}{q}
        \left(
            \log(K-k)+\log\frac{1}{\delta'}
        \right)
    \right\rceil .
    \label{eq:appendix-problem-independent-final}
\end{align}

Finally, $\mathbb{P}(\mathcal{E}_\delta^c)\le\delta$, while
Lemma~\ref{lem:appendix-empty-candidate-bound} ensures that
\begin{equation}
    \mathbb{P}\!\left(
        \mathcal{E}_\delta
        \cap
        \left\{
            N_T^{\mathrm{empty}}
            >
            \left\lceil
                \frac{K-k}{q}
                \left(
                    \log(K-k)+\log\frac{1}{\delta'}
                \right)
            \right\rceil
        \right\}
    \right)
    \le\delta'.
\end{equation}
A union bound therefore shows that
Eq.~\eqref{eq:appendix-problem-independent-final} holds with
probability at least $1-\delta-\delta'$, completing the proof.
\end{proof}

The problem-dependent bound instead uses the squared capped widths in
Eq.~\eqref{eq:appendix-gap-dependent-reduction}.

\begin{proof}[Proof of Theorem~\ref{thm:problem-dependent-regret}]
Again using $w_{t-1}(I_t)=2\beta a_t$,
\begin{equation}
    \min\{2,w_{t-1}(I_t)\}^2
    \le
    4\max\{\beta^2,1\}\min\{1,a_t^2\}.
    \label{eq:appendix-capped-width-squared}
\end{equation}
Summing over $t$ and applying
Eq.~\eqref{eq:appendix-elliptical-potential} gives
\begin{align}
    \sum_{t=1}^T\min\{2,w_{t-1}(I_t)\}^2
    &\le
    4\max\{\beta^2,1\}
    \sum_{t=1}^T\min\{1,a_t^2\}
    \nonumber\\
    &\le
    8\max\{\beta^2,1\}d
    \log\!\left(1+\frac{TL^2}{\lambda d}\right).
    \label{eq:appendix-gap-dependent-width-sum}
\end{align}

On $\mathcal{E}_\delta$, whenever the bound in
Lemma~\ref{lem:appendix-empty-candidate-bound} holds, combining
Eq.~\eqref{eq:appendix-gap-dependent-width-sum} with the
problem-dependent regret reduction
Eq.~\eqref{eq:appendix-gap-dependent-reduction} yields
\begin{align}
    R_T
    &\le
    \frac{1}{\Delta_{\min}}
    \sum_{t=1}^T
    \min\{2,w_{t-1}(I_t)\}^2
    +2N_T^{\mathrm{empty}}
    \nonumber\\
    &\le
    \frac{
        8\max\{\beta^2,1\}d
    }{\Delta_{\min}}
    \log\!\left(1+\frac{TL^2}{\lambda d}\right)
    \nonumber\\
    &\qquad
    +2
    \left\lceil
        \frac{K-k}{q}
        \left(
            \log(K-k)+\log\frac{1}{\delta'}
        \right)
    \right\rceil .
    \label{eq:appendix-problem-dependent-final}
\end{align}

As above, $\mathbb{P}(\mathcal{E}_\delta^c)\le\delta$, and
Lemma~\ref{lem:appendix-empty-candidate-bound} gives
\begin{equation}
    \mathbb{P}\!\left(
        \mathcal{E}_\delta
        \cap
        \left\{
            N_T^{\mathrm{empty}}
            >
            \left\lceil
                \frac{K-k}{q}
                \left(
                    \log(K-k)+\log\frac{1}{\delta'}
                \right)
            \right\rceil
        \right\}
    \right)
    \le\delta'.
\end{equation}
Hence, by a union bound,
Eq.~\eqref{eq:appendix-problem-dependent-final} holds with
probability at least $1-\delta-\delta'$, completing the proof.
\end{proof}

\section{Proof of the Access-Induced Regret Lower Bound}
\label{sec:appendix-uniform-challenger-lower-bound}

The lower bound follows from the waiting time until the initially inactive optimal action first enters the challenger set. Before this time, the candidate-access restriction prevents the policy from selecting the optimal action, regardless of its knowledge of the reward model.

\begin{proof}[Proof of Theorem~\ref{thm:uniform-gating-lower-bound}]
Set $N=K-k$. Recall that $\tau$ denotes the first round in which the initially inactive optimal action $i^\star$ appears in the challenger set. Since the next active set is chosen from the current candidate set, $i^\star$ remains inactive until round $\tau$. Conditional on $\mathcal{F}_{t-1}$ and $\{\tau>t-1\}$, it is therefore one of the $N$ inactive actions. Uniform sampling without replacement gives
\begin{equation}
    \mathbb{P}(\tau>t\mid\mathcal{F}_{t-1},\tau>t-1)
    =1-\frac{q}{N}.
    \label{eq:appendix-optimal-conditional-miss}
\end{equation}
Iterating Eq.~\eqref{eq:appendix-optimal-conditional-miss} yields
\begin{equation}
    \mathbb{P}(\tau>t)=\left(1-\frac{q}{N}\right)^t,
    \qquad t\ge0.
    \label{eq:appendix-optimal-geometric-tail}
\end{equation}

By the definition of the access-induced regret,
\begin{equation}
    R_T^{\mathrm{access}}
    =
    \Delta\sum_{t=1}^T\mathbf{1}\{\tau>t\}.
    \label{eq:appendix-access-regret}
\end{equation}
Moreover, Eq.~\eqref{eq:total-vs-access-regret} gives $R_T\ge R_T^{\mathrm{access}}$ pathwise. Taking expectations and using Eq.~\eqref{eq:appendix-optimal-geometric-tail},
\begin{align}
    \E[R_T]
    &\ge
    \E[R_T^{\mathrm{access}}]
    =
    \Delta\sum_{t=1}^T\mathbb{P}(\tau>t)
    \nonumber\\
    &=
    \Delta\sum_{t=1}^T\left(1-\frac{q}{N}\right)^t
    =
    \Delta\frac{N-q}{q}
    \left[1-\left(1-\frac{q}{N}\right)^T\right].
    \label{eq:appendix-exact-lower-bound}
\end{align}
Substituting $N=K-k$ establishes Eq.~\eqref{eq:exact-access-lower-bound}.

For the second claim, let $p=q/N$. If $q\le N/2$, then $p\le1/2$ and
\begin{equation}
    \frac{1-p}{p}\ge\frac{1}{2p},
    \qquad
    1-(1-p)^T\ge1-e^{-pT}.
    \label{eq:appendix-lower-bound-elementary-steps}
\end{equation}
Using $1-e^{-x}\ge(1-e^{-1})\min\{x,1\}$ for $x\ge0$,
\begin{align}
    \E[R_T^{\mathrm{access}}]
    &=
    \Delta\frac{1-p}{p}\left[1-(1-p)^T\right]
    \nonumber\\
    &\ge
    \frac{1-e^{-1}}{2}\,
    \Delta
    \min\left\{T,\frac{1}{p}\right\}
    \nonumber\\
    &=
    \frac{1-e^{-1}}{2}\,
    \Delta
    \min\left\{T,\frac{N}{q}\right\}.
    \label{eq:appendix-order-lower-bound}
\end{align}
Substituting $N=K-k$ proves Eq.~\eqref{eq:order-access-lower-bound} and completes the proof.
\end{proof}

\section{Experimental Details}
\label{sec:experimental-details}

\subsection{Data Construction}
\label{app:data_construction}

Synthetic instances follow the construction in Section~\ref{subsec:synthetic-experiments}. For the semi-synthetic benchmarks, we construct $d=128$ features from the full action set without using source scores and fit a separate reward parameter for each dataset.

\paragraph{MovieLens-20M.} Each action is a movie represented by its 1,128 Tag Genome relevance scores. We center the feature matrix and retain the first 128 principal components without whitening. We divide every feature vector by the maximum row norm of the full matrix, preserving their relative lengths. A user's observed movie ratings provide the source labels. We rank users by the number of ratings on Tag Genome movies and sample ten from the top 100 without replacement using seed 12345. The experiments use the first five sampled users: 12131, 91193, 27469, 50297, and 52009.

\paragraph{UCI SGEMM.} Each action is a GPU kernel configuration with 14 tuning parameters. We apply a base-two logarithm to the first ten parameters and standardize all 14 columns. The feature map contains these coordinates, their squares, all pairwise products, and nine products $z_jz_{j+1}z_{j+2}$ for $j=1,\ldots,9$. We standardize the resulting 128 columns and divide all rows by their maximum norm. The source score is the negative natural logarithm of the mean of four recorded runtimes for the GTX 680 task with matrix size 2048.

\paragraph{NAS-Bench-101.} Each action is an architecture evaluated on CIFAR-10 for 108 epochs. Features encode operations at each vertex, directed edges with endpoint operations, and input-to-output paths. Each token group is weighted by the inverse square root of its token count and mapped by a deterministic signed hash into 127 coordinates. One additional coordinate records graph size. Each feature vector is normalized to unit norm. The source score is the final validation accuracy averaged over three training repeats.

\paragraph{DOCKSTRING.} Each action is a molecule represented by a radius-two sparse Morgan count fingerprint computed from its SMILES representation. We apply a deterministic signed hash into 128 coordinates and normalize each feature vector to unit norm. We use PPARD, ABL1, ADAM17, ADRB1, and ADRB2, the first five targets in the source table. Source labels are negative docking scores; only finite labels enter regression, while all molecules remain available actions.

\paragraph{Licenses.} MovieLens-20M is distributed by GroupLens under its usage license for non-commercial research use. UCI SGEMM is licensed under CC BY 4.0. NAS-Bench-101 and DOCKSTRING are released under the Apache License 2.0.

\paragraph{Reward parameters.} For each dataset, let $\mathcal{J}$ be the labelled actions used for fitting, and let $\widetilde y_i$ denote the corresponding source scores standardized using this subset. We solve
\begin{equation}
    (\widehat b,\widehat v)
    \in\arg\min_{b,v}
    \left\{\sum_{i\in\mathcal{J}}(\widetilde y_i-b-x_i^\top v)^2+\lVert v\rVert_2^2\right\},
    \qquad
    \theta^\star=\frac{\widehat v}{\lVert\widehat v\rVert_2}.
\end{equation}
The offline ridge penalty is 1, and the intercept is unpenalized. Online means are $\mu_i=x_i^\top\theta^\star$, with independent Gaussian reward noise of standard deviation $0.05$. The bandit environments use all available labels for fitting.

\begin{table}[ht]
\centering
\caption{Held-out linearity diagnostics for the semi-synthetic action features. $R^2$ measures value prediction, and Spearman's $\rho$ measures ranking agreement. Entries with multiple instances report the mean $\pm$ sample standard deviation over five users or targets.}
\label{tab:semi-synthetic-linearity}
\footnotesize
\setlength{\tabcolsep}{6pt}
\renewcommand{\arraystretch}{1.08}
\begin{tabular}{@{}l r c c@{}}
\toprule
Dataset & $K$ & Held-out $R^2$ & Held-out $\rho$ \\
\midrule
MovieLens-20M & $10{,}381$ & $0.412\pm0.082$ & $0.638\pm0.069$ \\
UCI SGEMM & $241{,}600$ & $0.850$ & $0.886$ \\
NAS-Bench-101 & $423{,}624$ & $0.256$ & $0.749$ \\
DOCKSTRING & $260{,}155$ & $0.405\pm0.029$ & $0.627\pm0.031$ \\
\bottomrule
\end{tabular}
\end{table}

\paragraph{Linearity diagnostic.} We fit the ridge model on a random $20\%$ of the labelled actions and evaluate the unnormalized predictor, including its intercept, on the remaining $80\%$. Table~\ref{tab:semi-synthetic-linearity} reports the held-out diagnostics.

\subsection{Method Parameters}
\label{app:method_parameters}

All methods use the same bandit-level parameters specified below. For method-specific hyperparameters, we use fixed representative configurations rather than tuning each method separately for individual datasets or experimental settings. We follow the same principle for \textsc{PC-LinUCB}, whose candidate budget is fixed within each experimental setting. The exact configurations used in our experiments are reported below.

We set $\lambda=0.2$ and use the same fixed exploration coefficient from Eq.~\eqref{eq:confidence-radius} across all methods. With $R=0.05$, $L=S=1$, and $\delta=1/T$, this coefficient is approximately $1.0790$ for synthetic $d=16$, $1.9260$ for synthetic $d=100$, and $1.9242$ for the semi-synthetic benchmarks. This shared coefficient is used with each method's own confidence-score construction.

\paragraph{PC-LinUCB and RS-OFUL.} \textsc{PC-LinUCB} uses the candidate budgets specified in the main text. RS-OFUL recomputes its action when the design-matrix determinant increases by a factor greater than two; it updates its sufficient statistics after every observation.

\paragraph{CBSCFD and PSI-LinUCB.} CBSCFD uses sketch size $m=4$ for $d=16$ and $m=20$ for $d\in\{100,128\}$. PSI-LinUCB uses rank $r=4$ and $r=16$, respectively, in the shared-parameter online setting: one global state is updated after each selected action and reward. Both methods score the full action set using their compressed matrix representations.

\paragraph{Sub-Elim.} We use an HNSW shortlist backend to approximately retrieve the actions required by the Sub-Elim search step, with shortlist size 30, search parameter $\mathrm{ef}=100$, graph connectivity $M=16$, and construction parameter $\mathrm{ef}_{\mathrm{construction}}=200$. The confidence multiplier equals the shared exploration coefficient, $\eta=1/\sqrt{T}$, and the maximum stage is 8 for $T=20{,}000$ and 9 for $T=100{,}000$. The initial threshold is $-1$ in all settings.

\subsection{Spectral Bandit Application}
\label{app:spectral-experiment}

We construct a weighted Erd\H{o}s--R\'enyi graph with $K=10{,}000$ nodes and edge probability $0.03$, with edge weights sampled independently from $\mathrm{Uniform}[0,1)$. We use the $J=200$ lowest-frequency eigenvectors of the unnormalized graph Laplacian as the reduced spectral representation. The mean-reward function is generated from the first five spectral components, whose coefficients are sampled independently from $\mathcal{N}(0,1)$ and rescaled so that the maximum absolute mean reward is one. Online rewards include independent Gaussian noise with standard deviation $0.01$.

SpectralUCB uses spectral regularization $\lambda=0.01$ and the effective-dimension confidence construction, with $R=0.01$, $\delta=0.001$, and the exact smoothness norm of the synthetic reward function. Each repeat independently generates the graph and reward function, while all configurations within a repeat share the graph, spectral representation, reward function, and reward-noise sequence. Timing excludes the common graph construction and eigendecomposition.

\subsection{Timing}
\label{app:timing}

\paragraph{Paired runs.} All experiments use master seed 12345. Each synthetic configuration uses ten independently generated instances. For the semi-synthetic experiments, the five repeats correspond to five different MovieLens users or DOCKSTRING targets, and five repetitions of the fixed SGEMM and NAS-Bench-101 environments. Within each repeat, all methods share the same action features, reward parameter, mean rewards, and round-wise Gaussian noise sequence. Instance generation, initial candidate selection, noise generation, and method-specific randomness use separate random streams derived from the common seed.

\paragraph{Timing.} All methods run serially on the same server equipped with an AMD Ryzen Threadripper 7960X CPU (24 cores, 48 threads) and 251\,GiB of memory, with BLAS restricted to one thread. Their execution order rotates cyclically by repetition index. We measure elapsed wall-clock time using a monotonic high-resolution timer. The timed interval starts before method initialization and ends after the final online round. It includes model initialization, candidate-pool and clipped-score construction, method-specific search-index construction, action selection, pseudo-regret bookkeeping, reward observation, and model updates. Dataset loading, feature construction, reward-parameter fitting, noise generation, numerical-library warmup, initial active-set sampling, and output serialization occur outside the timed interval.

\paragraph{Reported statistics.} Cumulative pseudo-regret curves report the mean over repeats with bands of $\pm1$ standard error (SE), where the SE is the sample standard deviation divided by $\sqrt{n}$. Speedup is the ratio of mean LinUCB wall time to mean method wall time. For paired measurements $x_i$ and $y_i$ and ratio $r=\bar{x}/\bar{y}$, we estimate its SE as
\begin{equation}
    \operatorname{SE}(r)
    =
    \frac{
        \operatorname{SD}\!\left((x_i-r y_i)/\bar{y}\right)
    }{\sqrt{n}},
\end{equation}
where $\operatorname{SD}$ denotes the sample standard deviation. For speedup, $x_i$ is the LinUCB wall time and $y_i$ is the method wall time. For the ablations, the reported final-regret change is
\begin{equation}
    100\left(
        \frac{\bar{R}_{\mathrm{PC}}}{\bar{R}_{\mathrm{LinUCB}}}
        -1
    \right).
\end{equation}
Writing $r_R=\bar{R}_{\mathrm{PC}}/\bar{R}_{\mathrm{LinUCB}}$, its SE is
\begin{equation}
    100
    \frac{
        \operatorname{SD}\!\left(
        (R_{\mathrm{PC},i}-r_R R_{\mathrm{LinUCB},i})
        /\bar{R}_{\mathrm{LinUCB}}
        \right)
    }{\sqrt{n}}.
\end{equation}
All plotted error bars represent $\pm1$ SE. The held-out diagnostics in Table~\ref{tab:semi-synthetic-linearity} instead report sample standard deviations across users or targets.

\subsection{Detailed Timing and Speedup Results}

\label{app:detailed-speedup}

Tables~\ref{tab:timing-synthetic-d16}--\ref{tab:timing-spectral} report the absolute wall-clock times and speedups for all experiments. Initialization and online times are means across repeats; total times and speedups are reported as mean $\pm$ one standard error and ratio of means $\pm$ one paired standard error, respectively, using the definitions in Appendix~\ref{app:timing}. Initialization is measured in milliseconds, while online and total times are measured in seconds. Initialization time is the difference between the recorded total and online times. The online interval includes action selection, reward observation, regret bookkeeping, and model updates; these operations were not timed separately. Speedups are relative to LinUCB, or SpectralUCB for the spectral experiment. Shared data preparation is excluded. Each ablation configuration uses its own paired LinUCB runs. The main comparisons and ablations were timed in separate batches; each speedup uses the paired baseline from the same batch.

The spectral experiment additionally records a shared preparation time of $7.987\pm0.097$ seconds (mean $\pm$ one standard error over ten repeats), covering graph construction, eigendecomposition, feature and reward construction, noise generation, and initial-permutation generation. This shared cost is excluded from Table~\ref{tab:timing-spectral}; its individual components were not timed separately.

\begin{table}[htbp]
    \centering
    \small
    \setlength{\tabcolsep}{4pt}
    \caption{Synthetic comparison with $K=10{,}000$, $d=16$, $T=20{,}000$, and ten paired repeats. \textsc{PC-LinUCB} uses $k=q=250$.}
    \label{tab:timing-synthetic-d16}
    \begin{tabular}{lrrrr}
        \toprule
        Method & Init. (ms) & Online (s) & Total (s) & Speedup \\
        \midrule
        \textsc{PC-LinUCB} & 0.305 & 0.895 & $0.895 \pm 0.003$ & $5.07 \pm 0.02$ \\
        LinUCB & 0.040 & 4.540 & $4.540 \pm 0.007$ & $1.00 \pm 0.00$ \\
        RS-OFUL & 0.049 & 0.150 & $0.150 \pm 0.001$ & $30.29 \pm 0.22$ \\
        CBSCFD & 0.114 & 4.213 & $4.213 \pm 0.009$ & $1.08 \pm 0.00$ \\
        PSI-LinUCB & 0.130 & 5.401 & $5.401 \pm 0.018$ & $0.84 \pm 0.00$ \\
        Sub-Elim & 16.221 & 2.569 & $2.585 \pm 0.045$ & $1.76 \pm 0.03$ \\
        \bottomrule
    \end{tabular}
\end{table}

\begin{table}[htbp]
    \centering
    \small
    \setlength{\tabcolsep}{4pt}
    \caption{Synthetic comparison with $K=10{,}000$, $d=100$, $T=100{,}000$, and ten paired repeats. \textsc{PC-LinUCB} uses $k=q=250$.}
    \label{tab:timing-synthetic-d100}
    \begin{tabular}{lrrrr}
        \toprule
        Method & Init. (ms) & Online (s) & Total (s) & Speedup \\
        \midrule
        \textsc{PC-LinUCB} & 0.440 & 15.011 & $15.012 \pm 0.017$ & $15.50 \pm 0.16$ \\
        LinUCB & 0.042 & 232.675 & $232.675 \pm 2.313$ & $1.00 \pm 0.00$ \\
        RS-OFUL & 0.049 & 3.221 & $3.221 \pm 0.045$ & $72.25 \pm 0.76$ \\
        CBSCFD & 0.241 & 137.570 & $137.570 \pm 1.048$ & $1.69 \pm 0.01$ \\
        PSI-LinUCB & 0.224 & 142.663 & $142.663 \pm 1.940$ & $1.63 \pm 0.01$ \\
        Sub-Elim & 14.643 & 136.877 & $136.892 \pm 6.257$ & $1.70 \pm 0.07$ \\
        \bottomrule
    \end{tabular}
\end{table}

\begin{table}[htbp]
    \centering
    \small
    \setlength{\tabcolsep}{4pt}
    \caption{MovieLens-20M comparison with $K=10{,}381$, $d=128$, $T=20{,}000$, and five paired repeats. \textsc{PC-LinUCB} uses $k=q=104$.}
    \label{tab:timing-movielens20m_tag_genome}
    \begin{tabular}{lrrrr}
        \toprule
        Method & Init. (ms) & Online (s) & Total (s) & Speedup \\
        \midrule
        \textsc{PC-LinUCB} & 0.585 & 2.562 & $2.562 \pm 0.010$ & $38.66 \pm 0.11$ \\
        LinUCB & 0.037 & 99.057 & $99.057 \pm 0.175$ & $1.00 \pm 0.00$ \\
        RS-OFUL & 0.036 & 3.800 & $3.800 \pm 0.143$ & $26.07 \pm 1.01$ \\
        CBSCFD & 0.614 & 37.280 & $37.281 \pm 0.151$ & $2.66 \pm 0.01$ \\
        PSI-LinUCB & 0.582 & 42.502 & $42.502 \pm 0.573$ & $2.33 \pm 0.03$ \\
        Sub-Elim & 15.432 & 77.560 & $77.575 \pm 0.389$ & $1.28 \pm 0.01$ \\
        \bottomrule
    \end{tabular}
\end{table}

\begin{table}[htbp]
    \centering
    \small
    \setlength{\tabcolsep}{4pt}
    \caption{UCI SGEMM comparison with $K=241{,}600$, $d=128$, $T=20{,}000$, and five paired repeats. \textsc{PC-LinUCB} uses $k=q=2416$.}
    \label{tab:timing-uci_sgemm}
    \begin{tabular}{lrrrr}
        \toprule
        Method & Init. (ms) & Online (s) & Total (s) & Speedup \\
        \midrule
        \textsc{PC-LinUCB} & 15.488 & 42.573 & $42.588 \pm 0.940$ & $74.74 \pm 1.28$ \\
        LinUCB & 0.087 & 3182.931 & $3182.931 \pm 19.509$ & $1.00 \pm 0.00$ \\
        RS-OFUL & 0.091 & 138.903 & $138.903 \pm 0.989$ & $22.91 \pm 0.10$ \\
        CBSCFD & 16.804 & 914.613 & $914.630 \pm 9.855$ & $3.48 \pm 0.02$ \\
        PSI-LinUCB & 16.747 & 1096.383 & $1096.399 \pm 37.921$ & $2.90 \pm 0.09$ \\
        Sub-Elim & 373.000 & 1232.161 & $1232.534 \pm 5.321$ & $2.58 \pm 0.02$ \\
        \bottomrule
    \end{tabular}
\end{table}

\begin{table}[htbp]
    \centering
    \small
    \setlength{\tabcolsep}{4pt}
    \caption{NAS-Bench-101 comparison with $K=423{,}624$, $d=128$, $T=20{,}000$, and five paired repeats. \textsc{PC-LinUCB} uses $k=q=4237$.}
    \label{tab:timing-nasbench101}
    \begin{tabular}{lrrrr}
        \toprule
        Method & Init. (ms) & Online (s) & Total (s) & Speedup \\
        \midrule
        \textsc{PC-LinUCB} & 25.772 & 84.216 & $84.242 \pm 4.709$ & $66.05 \pm 3.94$ \\
        LinUCB & 0.127 & 5564.294 & $5564.294 \pm 34.611$ & $1.00 \pm 0.00$ \\
        RS-OFUL & 0.118 & 225.524 & $225.524 \pm 0.612$ & $24.67 \pm 0.18$ \\
        CBSCFD & 29.144 & 1556.777 & $1556.806 \pm 10.746$ & $3.57 \pm 0.02$ \\
        PSI-LinUCB & 28.680 & 1848.407 & $1848.436 \pm 37.106$ & $3.01 \pm 0.04$ \\
        Sub-Elim & 638.478 & 3459.211 & $3459.849 \pm 7.794$ & $1.61 \pm 0.01$ \\
        \bottomrule
    \end{tabular}
\end{table}

\begin{table}[htbp]
    \centering
    \small
    \setlength{\tabcolsep}{4pt}
    \caption{DOCKSTRING comparison with $K=260{,}155$, $d=128$, $T=20{,}000$, and five paired repeats. \textsc{PC-LinUCB} uses $k=q=2602$.}
    \label{tab:timing-dockstring}
    \begin{tabular}{lrrrr}
        \toprule
        Method & Init. (ms) & Online (s) & Total (s) & Speedup \\
        \midrule
        \textsc{PC-LinUCB} & 15.972 & 43.211 & $43.227 \pm 0.156$ & $77.92 \pm 0.42$ \\
        LinUCB & 0.091 & 3368.120 & $3368.120 \pm 10.575$ & $1.00 \pm 0.00$ \\
        RS-OFUL & 0.082 & 138.386 & $138.386 \pm 0.426$ & $24.34 \pm 0.08$ \\
        CBSCFD & 17.620 & 953.688 & $953.706 \pm 5.976$ & $3.53 \pm 0.03$ \\
        PSI-LinUCB & 17.721 & 1100.410 & $1100.428 \pm 17.285$ & $3.06 \pm 0.05$ \\
        Sub-Elim & 394.129 & 2516.168 & $2516.562 \pm 4.617$ & $1.34 \pm 0.00$ \\
        \bottomrule
    \end{tabular}
\end{table}

\begin{table}[htbp]
    \centering
    \small
    \setlength{\tabcolsep}{4pt}
    \caption{Candidate-budget ablation with $K=10{,}000$, $d=16$, $T=20{,}000$, and ten paired repeats per configuration. \textsc{PC-LinUCB} uses $k=q=B/2$.}
    \label{tab:timing-03_budget_ablation-d16}
    \begin{tabular}{llrrrr}
        \toprule
        Configuration & Method & Init. (ms) & Online (s) & Total (s) & Speedup \\
        \midrule
        $B/K=1\%$ & \textsc{PC-LinUCB} & 0.260 & 0.555 & $0.555 \pm 0.001$ & $8.23 \pm 0.03$ \\
        $B/K=1\%$ & LinUCB & 0.018 & 4.570 & $4.570 \pm 0.013$ & $1.00 \pm 0.00$ \\
        \addlinespace
        $B/K=2\%$ & \textsc{PC-LinUCB} & 0.265 & 0.639 & $0.639 \pm 0.001$ & $7.10 \pm 0.02$ \\
        $B/K=2\%$ & LinUCB & 0.017 & 4.537 & $4.537 \pm 0.009$ & $1.00 \pm 0.00$ \\
        \addlinespace
        $B/K=5\%$ & \textsc{PC-LinUCB} & 0.259 & 0.868 & $0.868 \pm 0.002$ & $5.23 \pm 0.01$ \\
        $B/K=5\%$ & LinUCB & 0.017 & 4.539 & $4.539 \pm 0.005$ & $1.00 \pm 0.00$ \\
        \addlinespace
        $B/K=10\%$ & \textsc{PC-LinUCB} & 0.256 & 1.242 & $1.242 \pm 0.006$ & $3.63 \pm 0.02$ \\
        $B/K=10\%$ & LinUCB & 0.019 & 4.511 & $4.511 \pm 0.006$ & $1.00 \pm 0.00$ \\
        \addlinespace
        $B/K=20\%$ & \textsc{PC-LinUCB} & 0.259 & 1.928 & $1.928 \pm 0.008$ & $2.35 \pm 0.01$ \\
        $B/K=20\%$ & LinUCB & 0.018 & 4.536 & $4.536 \pm 0.007$ & $1.00 \pm 0.00$ \\
        \bottomrule
    \end{tabular}
\end{table}

\begin{table}[htbp]
    \centering
    \small
    \setlength{\tabcolsep}{4pt}
    \caption{Candidate-budget ablation with $K=10{,}000$, $d=100$, $T=100{,}000$, and ten paired repeats per configuration. \textsc{PC-LinUCB} uses $k=q=B/2$.}
    \label{tab:timing-03_budget_ablation-d100}
    \begin{tabular}{llrrrr}
        \toprule
        Configuration & Method & Init. (ms) & Online (s) & Total (s) & Speedup \\
        \midrule
        $B/K=1\%$ & \textsc{PC-LinUCB} & 0.572 & 7.253 & $7.254 \pm 0.013$ & $49.18 \pm 0.10$ \\
        $B/K=1\%$ & LinUCB & 0.024 & 356.727 & $356.727 \pm 0.246$ & $1.00 \pm 0.00$ \\
        \addlinespace
        $B/K=2\%$ & \textsc{PC-LinUCB} & 0.572 & 11.969 & $11.969 \pm 0.024$ & $30.21 \pm 0.11$ \\
        $B/K=2\%$ & LinUCB & 0.023 & 361.569 & $361.570 \pm 1.784$ & $1.00 \pm 0.00$ \\
        \addlinespace
        $B/K=5\%$ & \textsc{PC-LinUCB} & 0.565 & 23.762 & $23.762 \pm 0.077$ & $15.16 \pm 0.07$ \\
        $B/K=5\%$ & LinUCB & 0.025 & 360.317 & $360.317 \pm 0.654$ & $1.00 \pm 0.00$ \\
        \addlinespace
        $B/K=10\%$ & \textsc{PC-LinUCB} & 0.677 & 43.872 & $43.873 \pm 0.099$ & $8.37 \pm 0.03$ \\
        $B/K=10\%$ & LinUCB & 0.031 & 367.107 & $367.107 \pm 1.999$ & $1.00 \pm 0.00$ \\
        \addlinespace
        $B/K=20\%$ & \textsc{PC-LinUCB} & 0.644 & 83.311 & $83.312 \pm 0.308$ & $4.47 \pm 0.02$ \\
        $B/K=20\%$ & LinUCB & 0.028 & 372.615 & $372.615 \pm 2.229$ & $1.00 \pm 0.00$ \\
        \bottomrule
    \end{tabular}
\end{table}

\begin{table}[htbp]
    \centering
    \small
    \setlength{\tabcolsep}{4pt}
    \caption{Budget-allocation ablation with $K=10{,}000$, $d=16$, $T=20{,}000$, and ten paired repeats per configuration. \textsc{PC-LinUCB} uses $B=k+q=500$.}
    \label{tab:timing-04_allocation_ablation-d16}
    \begin{tabular}{llrrrr}
        \toprule
        Configuration & Method & Init. (ms) & Online (s) & Total (s) & Speedup \\
        \midrule
        $(k,q)=(50,450)$ & \textsc{PC-LinUCB} & 0.269 & 0.891 & $0.891 \pm 0.002$ & $5.16 \pm 0.01$ \\
        $(k,q)=(50,450)$ & LinUCB & 0.021 & 4.601 & $4.601 \pm 0.010$ & $1.00 \pm 0.00$ \\
        \addlinespace
        $(k,q)=(125,375)$ & \textsc{PC-LinUCB} & 0.272 & 0.879 & $0.879 \pm 0.003$ & $5.23 \pm 0.02$ \\
        $(k,q)=(125,375)$ & LinUCB & 0.021 & 4.601 & $4.601 \pm 0.011$ & $1.00 \pm 0.00$ \\
        \addlinespace
        $(k,q)=(250,250)$ & \textsc{PC-LinUCB} & 0.266 & 0.869 & $0.870 \pm 0.001$ & $5.21 \pm 0.01$ \\
        $(k,q)=(250,250)$ & LinUCB & 0.021 & 4.534 & $4.534 \pm 0.005$ & $1.00 \pm 0.00$ \\
        \addlinespace
        $(k,q)=(375,125)$ & \textsc{PC-LinUCB} & 0.273 & 0.858 & $0.859 \pm 0.004$ & $5.30 \pm 0.02$ \\
        $(k,q)=(375,125)$ & LinUCB & 0.022 & 4.552 & $4.552 \pm 0.004$ & $1.00 \pm 0.00$ \\
        \addlinespace
        $(k,q)=(450,50)$ & \textsc{PC-LinUCB} & 0.287 & 0.834 & $0.835 \pm 0.003$ & $5.46 \pm 0.02$ \\
        $(k,q)=(450,50)$ & LinUCB & 0.023 & 4.553 & $4.553 \pm 0.005$ & $1.00 \pm 0.00$ \\
        \bottomrule
    \end{tabular}
\end{table}

\begin{table}[htbp]
    \centering
    \small
    \setlength{\tabcolsep}{4pt}
    \caption{Budget-allocation ablation with $K=10{,}000$, $d=100$, $T=100{,}000$, and ten paired repeats per configuration. \textsc{PC-LinUCB} uses $B=k+q=500$.}
    \label{tab:timing-04_allocation_ablation-d100}
    \begin{tabular}{llrrrr}
        \toprule
        Configuration & Method & Init. (ms) & Online (s) & Total (s) & Speedup \\
        \midrule
        $(k,q)=(50,450)$ & \textsc{PC-LinUCB} & 0.814 & 24.487 & $24.488 \pm 0.055$ & $15.60 \pm 0.09$ \\
        $(k,q)=(50,450)$ & LinUCB & 0.032 & 381.932 & $381.932 \pm 2.024$ & $1.00 \pm 0.00$ \\
        \addlinespace
        $(k,q)=(125,375)$ & \textsc{PC-LinUCB} & 0.609 & 23.652 & $23.653 \pm 0.131$ & $15.06 \pm 0.06$ \\
        $(k,q)=(125,375)$ & LinUCB & 0.029 & 356.097 & $356.097 \pm 2.725$ & $1.00 \pm 0.00$ \\
        \addlinespace
        $(k,q)=(250,250)$ & \textsc{PC-LinUCB} & 0.649 & 23.742 & $23.743 \pm 0.231$ & $15.27 \pm 0.15$ \\
        $(k,q)=(250,250)$ & LinUCB & 0.028 & 362.568 & $362.568 \pm 4.355$ & $1.00 \pm 0.00$ \\
        \addlinespace
        $(k,q)=(375,125)$ & \textsc{PC-LinUCB} & 0.601 & 23.396 & $23.396 \pm 0.179$ & $15.23 \pm 0.07$ \\
        $(k,q)=(375,125)$ & LinUCB & 0.027 & 356.435 & $356.435 \pm 2.755$ & $1.00 \pm 0.00$ \\
        \addlinespace
        $(k,q)=(450,50)$ & \textsc{PC-LinUCB} & 0.525 & 22.124 & $22.124 \pm 0.024$ & $15.44 \pm 0.05$ \\
        $(k,q)=(450,50)$ & LinUCB & 0.022 & 341.670 & $341.670 \pm 1.197$ & $1.00 \pm 0.00$ \\
        \bottomrule
    \end{tabular}
\end{table}

\begin{table}[htbp]
    \centering
    \small
    \setlength{\tabcolsep}{4pt}
    \caption{Spectral comparison with $K=10{,}000$, $J=200$, $T=50{,}000$, and ten paired repeats. PC-SpectralUCB uses $k=q=50$, $250$, and $500$ at budgets of $1\%$, $5\%$, and $10\%$, respectively.}
    \label{tab:timing-spectral}
    \begin{tabular}{lrrrr}
        \toprule
        Method & Init. (ms) & Online (s) & Total (s) & Speedup \\
        \midrule
        SpectralUCB & 0.050 & 736.438 & $736.438 \pm 2.478$ & $1.00 \pm 0.00$ \\
        PC-SpectralUCB (1\%) & 1.247 & 10.490 & $10.491 \pm 0.039$ & $70.20 \pm 0.30$ \\
        PC-SpectralUCB (5\%) & 1.301 & 35.650 & $35.652 \pm 0.072$ & $20.66 \pm 0.08$ \\
        PC-SpectralUCB (10\%) & 1.294 & 67.213 & $67.215 \pm 0.072$ & $10.96 \pm 0.04$ \\
        \bottomrule
    \end{tabular}
\end{table}

\end{document}